\documentclass[wcp]{jmlr}
\usepackage{amsmath,amssymb,graphicx,url}
\jmlrvolume{204}
\jmlryear{2023}
\jmlrworkshop{Conformal and Probabilistic Prediction with Applications}
\jmlrproceedings{PMLR}{Proceedings of Machine Learning Research}

\usepackage{longtable}

\usepackage{booktabs}
 \usepackage{siunitx}

\usepackage{bbm}
\usepackage{algorithm}
\usepackage{algorithmic}
\usepackage{wrapfig}
\usepackage{mathtools}
\usepackage{float}
\usepackage{placeins}

\usepackage{amssymb}
\usepackage{enumitem}
\usepackage{multirow}

\newcommand{\cY}{\mathcal{Y}}

\newcommand{\hatp}{\hat{p}}

\newcommand{\argmin}{\operatorname*{argmin}}
\newcommand{\argmax}{\operatorname*{argmax}}

\newcommand{\given}{\, | \,}

\newcommand{\fromto}{\longrightarrow}

\renewcommand{\to}{\longrightarrow}
\newcommand{\phat}{\hat{p}}

\renewcommand{\vec}[1]{\boldsymbol{#1}}

\newcommand{\ccssl}{CCSSL}
\newcommand{\kldiv}[2]{D_{KL}(#1 \, || \, #2)}

\usepackage[capitalize]{cleveref}
\crefname{section}{Sec.}{Secs.}
\Crefname{section}{Section}{Sections}
\Crefname{table}{Table}{Tables}
\crefname{table}{Tab.}{Tabs.}

\theorembodyfont{\upshape}
\theoremheaderfont{\scshape}
\theorempostheader{:}
\theoremsep{\newline}

\title{Conformal Credal Self-Supervised Learning}

  \author{\Name{Julian Lienen\nametag{\thanks{Corresponding author}}} \Email{julian.lienen@upb.de}\\
   \Name{Caglar Demir} \Email{caglar.demir@upb.de}\\
   \addr{Paderborn University, Germany}
   \AND
   \Name{Eyke H\"ullermeier} \Email{eyke@lmu.de}\\
   \addr{University of Munich (LMU), Germany}
   }

\editor{Harris Papadopoulos, Khuong An Nguyen, Henrik Boström and Lars Carlsson}

\begin{document}




\maketitle

\begin{abstract}
   In semi-supervised learning, the paradigm of self-training refers to the idea of learning from pseudo-labels suggested by the learner itself. Recently, corresponding methods have proven effective and achieve state-of-the-art performance, e.g., when applied to image classification problems. However, pseudo-labels typically stem from ad-hoc heuristics, relying on the quality of the predictions though without guaranteeing their validity. One such method, so-called credal self-supervised learning, maintains pseudo-supervision in the form of sets of (instead of single) probability distributions over labels, thereby allowing for a flexible yet uncertainty-aware labeling. Again, however, there is no justification beyond empirical effectiveness. To address this deficiency, we make use of conformal prediction, an approach that comes with guarantees on the validity of set-valued predictions. As a result, the construction of credal sets of labels is supported by a rigorous theoretical foundation, leading to better calibrated and less error-prone supervision for unlabeled data. Along with this, we present effective algorithms for learning from credal self-supervision. An empirical study demonstrates excellent calibration properties of the pseudo-supervision, as well as the competitiveness of our method on several image classification benchmark datasets.
\end{abstract}

\section{Introduction}
\label{sec:intro}

In the recent years, machine learning applications, particularly deep learning models, have benefited greatly from the extensive amount of available data. While traditional supervised learning methods commonly rely on curated, precise target information, there is an essentially unlimited availability of unlabeled data, e.g., crowd-sourced images in computer vision applications. Semi-supervised learning \citep{chapelle2009semi} aims to leverage this source of information for further optimizing models, as successfully demonstrated by a wide range of algorithms \citep{mixmatch,remixmatch,fixmatch,UDA}. Among other approaches, so-called \textit{self-training} \citep{lee2013pseudo} emerged as a simple yet effective paradigm to facilitate learning from additional unlabeled data. To this end, the model to be trained is used to predict ``pseudo-labels'' for the unlabeled data, which are then added to the labeled training data.

When learning probabilistic classifiers, e.g., to classify images, pseudo-labels can come in different forms. Traditionally, such supervision is typically in the form of a probability distribution, either degenerate (``one-hot'' encoded, hard) \citep{lee2013pseudo,fixmatch} or non-degenerate (soft) \citep{mixmatch,remixmatch}. As pointed out by \citet{lienen2021label}, this form can be questioned from a data modeling perspective: Not only is a single distribution unlikely to match the true ground-truth, and may hence to bias the learner, but also incapable of expressing uncertainty about the ground-truth distribution. 
While uncertainty-aware approaches address these shortcomings by accompanying the targets with additional qualitative information (e.g., as in \citep{DBLP:conf/iclr/RizveDRS21}), so-called \textit{credal self-supervised learning} (CSSL) \citep{cssl} replaces the targets by \textit{credal sets}, i.e., sets of probability distributions assumed to guarantee (or at least likely) covering the true distribution. Such set-valued targets, for which CSSL employs a generalized risk minimization to enable learning, relieve the learner from committing to a single distribution, while facilitating uncertainty-awareness through increased expressivity.

Although all of the aforementioned approaches have demonstrated their effectiveness to the problem of semi-supervised learning, none of them entails meaningful guarantees about the validity of the pseudo-labels\,---\,their quality is solely subject to the model confidence itself without an objective error probability. 
Moreover, rather ad-hoc heuristics are considered to control the influence of vague beliefs on the learning process (e.g., by selecting pseudo-labels surpassing a confidence threshold). For instance, CSSL derives credal pseudo-supervision based on the predicted probability scores.
Although it exhibits uncertainty-awareness by using entropy minimization (less entropy results in smaller sets), this may be misleading for overconfident yet poorly generalizing models. Again, no objective guarantee can be given for the validity of the credal set. This raises the quest for a pseudo-labeling procedure providing exactly that kind of guarantee, i.e., validating pseudo-supervision beyond mere empirical effectiveness.

Fortunately, the framework of conformal prediction (CP) \citep{vovk05,DBLP:journals/jmlr/ShaferV08} provides a tool-set to satisfy such demands.
As such, CP induces prediction regions quantifying the uncertainty for candidate values as outcome for a query instance by measuring their non-conformity (``strangeness'') with previously observed data with known outcomes.
Much like in classical hypothesis testing, it computes p-values for candidate outcomes (e.g., for each class $y'$) as a criterion for hypothesis rejection, i.e., whether to include a value $y'$ in the prediction set with a certain amount of confidence. Based on very mild technical assumptions, CP provides formal guarantees for the validity of prediction regions and is able to control the probability of an invalid prediction (not covering the true label).

In our work, we combine the technique of inductive conformal prediction, an efficient variant of CP, with the idea of credal self-supervised learning, leading to a method we dub \textit{conformal credal self-supervised learning} (\ccssl). Instead of deriving credal sets based on an ad-hoc estimate of the model confidence, we proceed from a possibilistic interpretation of conformal predictions \citep{DBLP:journals/ijar/CellaM22} to construct conformal credal labels. By this, pseudo-labels of that kind provide the validity guarantees implied by conformal predictors, laying a rigorous theoretical foundation for a set-valued pseudo-labeling strategy and thus paving the way for more thorough theoretical analyses of self-training in semi-supervised learning. To enable learning from conformal credal self-supervision, we further provide an effective algorithmic solution for generalized empirical risk minimization on set-valued targets.
An exhaustive empirical study on several image classification datasets demonstrates the usefulness of our method for effective and reliable semi-supervised learning\,---\,leveraging well calibrated pseudo-supervision, it is shown to improve the state-of-the-art in terms of generalization performance.

\section{Related Work}
\label{sec:related_work}

Semi-supervised learning describes the learning paradigm where the aim is to leverage the potential of unlabeled in addition to labeled data to improve learning and generalization. As it is easy to lose track of related work due to the wide range of proposed approaches in the recent years, we focus here on classification methods applied to computer vision applications that are closely related to our work. For a more comprehensive overview, we refer to \citep{chapelle2009semi} and \citep{DBLP:journals/ml/EngelenH20}.

Nowadays, so-called \textit{self-training} (or self-supervised) methods protrude among these methods by providing a simple yet effective learning methodology to make use of unlabeled data. Such methods can be found in a wide range of domains, including computer vision \citep{DBLP:conf/wacv/RosenbergHS05,Doersch2016,DBLP:conf/iccv/GodardAFB19,DBLP:conf/cvpr/XieLHL20} and natural language processing \citep{DBLP:conf/naacl/DuGGCCASC21}. As self-training can be considered as a general learning paradigm, where a model suggests itself labels to learn from, it has been wrapped around various model types, e.g., support vector machines \citep{Lienen2021instweighting}, decision trees \citep{Tanha2017} and most prominently with neural networks \citep{DBLP:conf/nips/OliverORCG18}. Notably, it lays the foundation for so-called distillation models, e.g., in self-distillation \citep{DBLP:conf/iccv/KimJYH21} or student-teacher setups \citep{DBLP:conf/cvpr/XieLHL20,DBLP:conf/cvpr/PhamDXL21}. It is further popular for unsupervised pretraining, e.g., as described in \citep{NEURIPS2020_f3ada80d,DBLP:conf/iccv/CaronTMJMBJ21}. \looseness=-1

Uncertainty-awareness \citep{DBLP:journals/ml/HullermeierW21} is a critical aspect for a cautious pseudo label selection to learn from, without exposing the model to the risk of confirmation biases \citep{softlabelsconfbias}. This aspect has been considered in previous works by different means. \citep{DBLP:conf/iclr/RizveDRS21} employs Bayesian sampling techniques, such as MC-Dropout \citep{mcdropout} or DropBlock \citep{DBLP:conf/nips/GhiasiLL18}, to estimate uncertainty of a prediction, which is then used as an additional filter criterion. \citep{DBLP:conf/nips/RenYS20} uses an adaptive instance weighting approach to control the influence of individual pseudo labels. Moreover, credal self-supervised learning \citep{cssl} expresses certainty by the size of maintained credal sets as pseudo-supervision. Also, learning from softened probability distribution can also be considered as a way to suppress overconfidence tendencies to learn in a more cautious way \citep{mixmatch,softlabelsconfbias}. Lastly, several domain-specific adaptions have been proposed, e.g., for text classification \citep{DBLP:conf/nips/MukherjeeA20} or semantic segmentation \citep{DBLP:journals/ijcv/ZhengY21}.

As prominently used throughout this work, conformal prediction \citep{vovk05,DBLP:journals/jmlr/ShaferV08,conf_pred} provides an elegant framework to naturally express model prediction uncertainty in a set-valued form. We refer the interested reader to  \citep{DBLP:journals/corr/abs-2005-07972} for a more comprehensive overview beyond the covered literature here.
Whereas conformal prediction initially proceeded from a transductive form in an online setting \citep{vovk05}, its high complexity called for more efficient variants. To this end, several variations have been proposed, where inductive conformal prediction \citep{DBLP:conf/ictai/PapadopoulosVG07} most prominently draw attention by assuming a separated calibration split available used for non-conformity measurement. Beyond this, several alternative approaches varying data assumption or algorithmic improvements have been proposed (e.g., as described in \citep{lei2018distribution,DBLP:conf/nips/KimXB20}). Recently, \citet{DBLP:conf/isipta/CellaM21,DBLP:journals/ijar/CellaM22} suggest a reinterpretation of conformal transducer as plausibility contours, allowing to derive validity guarantees on predictive distributions rather than prediction regions, which is being used throughout our work.

\section{Conformal Credal Pseudo-Labeling}
\label{sec:conf_pl}

In this section, we revisit credal pseudo-labeling and introduce conformal prediction as a framework for reliable prediction and an integral part of our conformal credal pseudo-labeling approach.

\subsection{Credal Pseudo-Labeling}
\label{sec:conf_pl:credal_labeling}

In supervised classification, one typically assumes instances $\vec{x} \in \mathcal{X}$ of an instance space $\mathcal{X}$ to be associated with a ground-truth in the form of a conditional probability distribution $p^*(\cdot \given \vec{x}) \in \mathbb{P}(\mathcal{Y})$ over the class space $\mathcal{Y}=\{y_1, ..., y_K\}$. Training a probabilistic classifier $\phat : \mathcal{X} \fromto \mathbb{P}(\mathcal{Y})$ commonly involves the optimization of a probabilistic surrogate loss function  $\mathcal{L} : \mathbb{P}(\mathcal{Y}) \times \mathbb{P}(\mathcal{Y}) \fromto \mathbb{R}_+$ (e.g., cross-entropy) comparing the model prediction $\phat$ to a proxy distribution $p$ reflecting the true target $p^*(\cdot \given \vec{x})$. Although it would be desirable, most classification datasets do not give direct access to $p^*$, but only to a realization $y\in \mathcal{Y}$ of the random variable $Y \sim p^*(\cdot \given \vec{x})$, whence a degenerate distribution $p_y$ with $p_y(y) = 1$ and $p_y(y')=0$ for $y' \neq y$ is often considered as a proxy. Other data modeling techniques such as label smoothing \citep{szegedy_ls} also consider ``softened'' versions of approaching $p_y$. \looseness=-1 

What methods of this kind share is their reliance on a single probability distribution as target information. As already said, this information does not allow for representing any uncertainty about the ground-truth distribution $p^*$: Predicting a precise distribution $\hat{p}$, the learner pretends a level of certainty that is not warranted. This is a sharp conflict with the observation that, in practice,  
such predictions tend to be poorly calibrated and overconfident \citep{mueller_ls}, and are hence likely to bias the learner. \looseness=-1

To overcome these shortcomings, \citet{lienen2021label} suggest to the replace a single probabilistic prediction by a \textit{credal set} $\mathcal{Q}\subseteq \mathbb{P}(\mathcal{Y})$, i.e., as set of (candidate) probability distributions. This set $\mathcal{Q}$ is supposed to cover the ground-truth $p^*$, very much like a confidence interval is supposed to cover a ground-truth parameter in classical statistics. Thus, the learner is able to represent its uncertainty about the ground-truth $p^*$ in a more faithful way.
More specifically, the learner's (epistemic) uncertainty is in direct correspondence with the size of the credal set. 
It can represent complete ignorance about $p^*$ (by setting $\mathcal{Q}=\mathbb{P}(\mathcal{Y})$), complete certainty ($\mathcal{Q}=\{p\}$), but also epistemic states in-between these extremes.

One way to specify credal sets is by means of so-called \textit{possibility distributions} \citep{Dubois2004PossibilityTP} $\pi : \mathcal{Y} \fromto [0,1]$, which 
induce a possibility measure $\Pi : 2^\mathcal{Y} \fromto [0,1]$ by virtue of $\Pi(Y) = \max_{y \in Y} \pi(y)$ for all $Y \subseteq \mathcal{Y}$. Possibility degrees can be interpreted as upper probabilities, so that 
a possibility distribution $\pi$ specifies the following set of probability distributions:
\begin{equation}
    \label{eq:credal_set}
    \begin{split} 
    \mathcal{Q}_\pi := & \Big\{ p \in \mathbb{P}(\mathcal{Y}) \given \forall Y \subseteq \mathcal{Y}: \\&  P(Y) = \sum_{y \in Y} p(y) \leq \max_{y \in Y} \pi(y) = \Pi(Y) \Big\} \enspace .
    \end{split}
\end{equation}
To guarantee $\mathcal{Q}_\pi \neq \emptyset$, the distribution $\pi$ is normalized such that $\max_{y \in \mathcal{Y}} \pi(y) = 1$, i.e., there is at least one label $y$ considered fully plausible. 

When considering the setting of semi-supervised learning, credal labeling is employed as a pseudo labeling technique within the framework of credal self-supervised learning (CSSL) \citep{cssl}: For each unlabeled instance, a credal set is maintained to express the current belief about $p^*$ in terms of a confidence region $\mathcal{Q} \subseteq \mathbb{P}(\mathcal{Y})$. To this end, CSSL derives a possibility distribution $\pi$ by an ad-hoc heuristic that assigns full plausibility $\pi(\hat{y}) = 1$ to the class $\hat{y}$ for which the predicted probability is highest, and determines a constant plausibility degree $\alpha$ for all other classes $y' \neq \hat{y}$; the latter depends on the learner's confidence as well as the class prior and the prediction history. 

Although CSSL has proven to provide competitive generalization performance, the credal set construction heuristic lacks a solid theoretical foundation and does not provide any quality guarantees.
Especially in the case of overconfident models, the credal sets may not reflect uncertainty properly and misguide the learner. Moreover, the class prior and the prediction history employed in the specification of possibility distribution constitute yet another potential source of bias in the learning process. This raises the question whether credal pseudo-labeling can be based on a more solid theoretical foundation and equipped with validity guarantees. An affirmative answer to this question is offered by the framework of conformal prediction.

\subsection{Inductive Conformal Prediction}

\textit{Conformal prediction} (CP) \citep{vovk05} is a distribution-free uncertainty quantification framework that judges a prediction for a query by its ``non-conformity'' with data observed before. While it has been originally introduced and extensively analyzed in an online setting, we refer to an offline (``batch-wise'') variant as typically considered in supervised learning settings. 
Several alternative variants of conformal prediction have been proposed in the past, most notably transductive and inductive methods. In its original formulation, the former requires a substantial amount of training and is clearly unsuitable in large-scale learning scenarios such as typical deep learning applications. This issue has been addressed in so-called \textit{inductive conformal prediction} (ICP) \citep{papadopoulos2008inductive} by alleviating the computational demands through additional calibration data.

Assume we are given training data $\mathcal{D}_{\text{train}} = \{ (\vec{x}_i, y_i) \}_{i=1}^N \subset (\mathcal{X} \times \mathcal{Y})^N$ and calibration data $\mathcal{D}_{\text{calib}} = \{ (\vec{x}_i, y_i) \}_{i=1}^L \subset (\mathcal{X} \times \mathcal{Y})^L$ comprising $N$ resp.\ $L$ i.i.d.\ observations. 
For a given query (resp. test) instance $\vec{x}_{N+1}$, the goal of conformal prediction is to provide an uncertainty quantification with provable guarantees about the likeliness of each possible candidate $\hat{y} \in \mathcal{Y}$ being the true outcome associated with $\vec{x}_{N+1}$. This is done by measuring the conformity of $(\vec{x}_{N+1}, \hat{y})$ with $\mathcal{D}_{\text{calib}}$ for each candidate label $\hat{y}$.

To this end, conformal prediction calculates a \textit{non-conformity measure} $\alpha : (\mathcal{X} \times \mathcal{Y})^{m} \times (\mathcal{X} \times \mathcal{Y}) \fromto \mathbb{R}$, which indicates how well a pair of a query instance and a candidate label conforms to an observed sequence of $m$ pairs. Typically, the non-conformity $\alpha(\mathcal{D}, (\vec{x}, y))$ involves the training of an underlying predictor (a model $\phat$ in our case) on $\mathcal{D}$, whose prediction $\phat(\vec{x})$ is then compared to $y$. In a probabilistic learning scenario, common choices are
\begin{equation}
    \label{eq:non_conformity}
    \alpha(\mathcal{D}, (\vec{x}, y)) = \max_{y_j \neq y} \hatp_\mathcal{D}(\vec{x})(y_j) - \hatp_\mathcal{D}(\vec{x})(y)
\end{equation}
and
\begin{equation}
    \label{eq:non_conformity2}
    \alpha(\mathcal{D}, (\vec{x}, y)) = \frac{\max_{y_j \neq y} \hatp_\mathcal{D}(\vec{x})(y_j)}{\hatp_\mathcal{D}(\vec{x})(y) + \gamma} \enspace ,
\end{equation}
where $\phat_\mathcal{D}$ is a probabilistic classifier trained on $\mathcal{D}$, $\phat(\cdot)(y')$ denotes the output probability for class $y'$ and $\gamma \geq 0$ is a sensitivity parameter \citep{DBLP:conf/ictai/PapadopoulosVG07}. In our case, we set $\mathcal{D}$ to the training data $\mathcal{D}_{\text{train}}$.

In ICP, non-conformity scores $\alpha_i$ are calculated for all $(\vec{x}_i, y_i) \in \mathcal{D}_{\text{calib}}$, i.e., it is determined for each calibration instance how well it conforms with the underlying training data $\mathcal{D}_{\text{train}}$. For the uncertainty quantification of the prediction for a query instance $\vec{x}_{N+1}$, the non-conformity $\alpha^{\hat{y}}_{N+1} := \alpha(\mathcal{D}_{\text{train}}, (\vec{x}_{N+1}, \hat{y}))$ is calculated for each candidate label $\hat{y} \in \mathcal{Y}$. Given the non-conformity values $\alpha_1, \ldots, \alpha_L$ associated with the calibration data, these non-conformity values can be used to construct \textit{p-values} for each candidate label $\hat{y}$, comparable to the notion of p-values in traditional statistics:
\begin{equation}
    \label{eq:pvalue}
    \pi_{N+1}(\hat{y}) = \frac{|\{ \alpha_i \geq \alpha^{\hat{y}}_{N+1} \given i \in \{1, ..., L\} \}| + 1}{L+1} \enspace .
\end{equation}

Consequently, one can define a conformal predictor $\Gamma_\delta$ with confidence $1-\delta$ by
\begin{equation}
    \Gamma_\delta(\mathcal{D}_{\text{calib}}, \vec{x}_{N+1}) = \Big\{ \hat{y} \in \mathcal{Y}: \pi_{N+1}(\hat{y}) \geq \delta \Big\}.
\end{equation}
Such a $\Gamma_\delta$ can be shown to cover the true label $y_{N+1}$ associated with $\vec{x}_{N+1}$ with high probability:
\begin{equation}
    \label{eq:weak_validity}
    \Pr \Big(y_{N+1} \in \Gamma_\delta(\mathcal{D}_{\text{calib}}, \vec{x}_{N+1}) \big) \geq 1 - \delta \enspace ,
\end{equation}
where the probability is taken over $\vec{x}_{N+1}$ and $\mathcal{D}_{\text{calib}}$. This property is typically referred to as (marginal) \textit{validity}, however, following \citep{DBLP:journals/ijar/CellaM22}, we refer to (\ref{eq:weak_validity}) as \textit{weak validity}. Note that this holds for any underlying probability distribution, any $\delta \in (0,1)$, and $N \in \mathbb{N}_+$. Nevertheless, practically speaking, small calibration datasets $\mathcal{D}_{\text{calib}}$ may be problematic if a certain granularity in the confidence degree is desired \citep{DBLP:conf/slds/JohanssonABCLS15}. 

\subsubsection{Conformalized Predictive Distributions}

While the notion of weak validity applies to set-valued prediction regions, its application as quantification of \textit{predictive distributions} is not obvious. In the realm of probabilistic classification (as considered here), one may wonder how CP can be applied to provide a meaningful uncertainty quantification. Recently, \citet{DBLP:conf/isipta/CellaM21} proposed an interpretation of p-values $\pi$ (cf. \cref{eq:pvalue}) in terms of possibility degrees, i.e., upper probabilities of the event that the respective candidate label is the true outcome. Consequently, these possibilities $\pi$ induce possibility measures $\Pi$, such that
\begin{equation}
    \Pi(A) := \sup_{\hat{y} \in A} \pi(\hat{y}), \qquad A \subseteq \mathcal{Y} \, .
\end{equation}
With $\sup_{\hat{y} \in \mathcal{Y}} \pi(\hat{y}) = 1$ for all $\mathcal{D}_{\text{train}}$ and assuming that $\pi$ is stochastically no smaller than the uniform distribution $\mathcal{U}(0,1)$ under any underlying (ground-truth) probability distribution, probabilistic predictors satisfy the \textit{strong validity} property defined as follows:
\begin{equation}
    \label{eq:strong_validity}
    \Pr \big(\Pi(A) \leq \delta, y_{N+1} \in A \big) \leq \delta 
\end{equation}
holds for any $\delta \in (0,1)$, $A\subseteq \mathcal{Y}$, training data $\mathcal{D}_{\text{train}}$ and any underlying true probability distribution.

\subsection{Conformal Credal Pseudo-Labeling}
\label{sec:conf_pl:conformal_credal_pl}

Strongly valid possibility distributions $\pi$ of this kind can be directly employed in the credal set formulation in (\ref{eq:credal_set}) to bound the space of probability distributions. Relating to the self-training paradigm we consider in a semi-supervised setting, where a model suggests itself credal sets as pseudo-supervision to learn from, we refer to such credal sets as \textit{conformal credal pseudo-labels}. Note that such pseudo-labels do not necessarily need to be constructed by an inductive conformal prediction procedure, but can used any CP methodology with the same guarantees.

To guarantee property (\ref{eq:strong_validity}) for the possibility distribution $\pi$  defined in (\ref{eq:pvalue}), this distribution needs to be normalized such that $\max_{\hat{y} \in \mathcal{Y}} \pi(\hat{y}) = 1$ to ensure strong validity and that $\mathcal{Q}_\pi \neq \emptyset$. 
A commonly applied normalization is suggested in \citep{DBLP:conf/isipta/CellaM21}:
\begin{equation}
    \label{eq:norm1}
    \pi(\hat{y}) = \frac{\pi(\hat{y})}{\max_{y' \in \mathcal{Y}} \pi(y')} \enspace .
\end{equation}

By the described conformal credal labeling method, we replace the ad-hoc construction of sets in CSSL by a more profound technique with formal guarantees. More precisely, we build the learner's uncertainty-awareness on top of an \textit{objective} quality criterion deduced from its accuracy on the calibration data. This not only provides strong validity guarantees, but also paves the path to a more rigorous theoretical analysis of self-supervised credal learning. Nevertheless, these guarantees come at a cost: By applying a conformal prediction procedure, one either provokes an increased computational overhead (for instance in a transductive CP setting) or a decrease in data efficiency by requiring additional calibration data.
However, as will be seen in the empirical evaluation, the improved pseudo-label quality makes up for this in terms of generalization performance.

\section{Conformal Credal Self-Supervised Learning}
\label{sec:ccssl}

In the following, we provide an overview of our methodology for effective semi-supervised learning from conformal credal pseudo-labels.

\subsection{Overview}

\begin{figure}[htbp]
    \centering
    \includegraphics[trim={0.5cm 1.5cm 1cm 0.9cm},clip,width=0.9\textwidth]{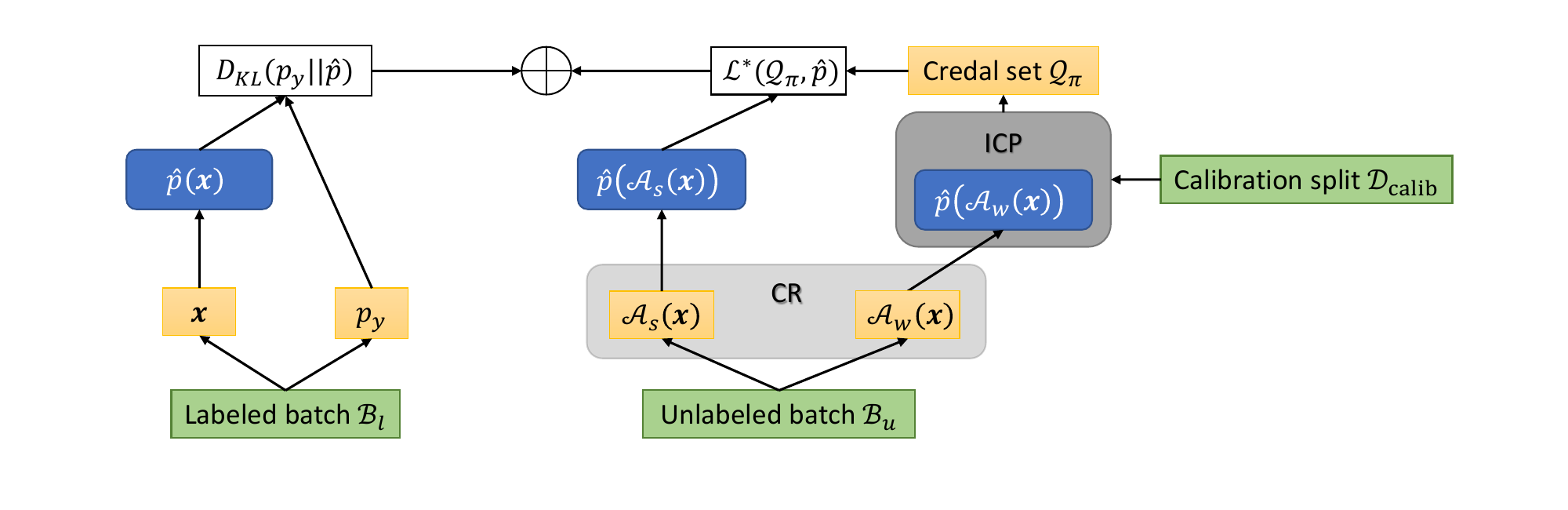}
    \caption{Overview over \ccssl: While learning on the labeled batch $\mathcal{B}_l$ is performed conventionally, the labels for the unlabeled batch $\mathcal{B}_u$ are constructed by an inductive conformal prediction (ICP) procedure provided calibration data $\mathcal{D}_{\text{calib}}$, employing consistency regularization (CR) for confirmation bias mitigation.}
    \label{fig:ccssl_overview}
\end{figure}

In our approach, which we call \textit{conformal credal self-supervised learning} (\ccssl), we combine credal self-supervised learning as proposed in \citep{cssl} with conformal credal labels. We thus replace the credal set construction in CSSL to ensure guarantees of the pseudo-label's validity. Although both CSSL and \ccssl{} are not specifically tailored to a particular domain, we adopt the framework of the former with a focus on image classification here.
Fig.\ \ref{fig:ccssl_overview} gives an overview of the algorithm, its pseudo-code can be found in the appendix.\footnote{The appendix, along with the official implementation, is externally available at \url{https://github.com/julilien/C2S2L}.}

In each iteration, we observe batches of labeled $\mathcal{B}_l = \{(\vec{x}_i, p_{y_i})\}_{i=1}^B$ and unlabeled instances $\mathcal{B}_u = \{\vec{x}_i\}_{i=1}^{\mu B}$, where $\mu \geq 1$ is the multiplicity of unlabeled over labeled instances in each batch. Here, we consider degenerate distributions $p_{y_i}$ as target information in $\mathcal{B}_l$. For the instances $(\vec{x}, p_y) \in \mathcal{B}_l$, we compute the Kullback-Leibler divergence $\kldiv{p_y}{\hat{p}(\vec{x})}$ as loss $\mathcal{L}_l$. 

To calculate the loss on unlabeled instances $\vec{x} \in \mathcal{B}_u$ leveraging our conformal credal labeling approach, we adopt the \textit{consistency regularization} (CR) framework \citep{consistency_reg,DBLP:conf/nips/SajjadiJT16}, which aims to mitigate the so-called \textit{confirmation bias} \citep{softlabelsconfbias} that describes the manifestation of misbeliefs of an accurate learner. CR enforces consistent predictions for different perturbed appearances of an instance, and has been successfully combined with pseudo-labeling, whereby \citet{fixmatch} propose a particularly simple scheme: An unlabeled instance $\vec{x} \in \mathcal{X}$ is augmented by a weak (e.g., horizontal flipping) and a strong (e.g., RandAugment \citep{randaugment} or CTAugment \citep{remixmatch}) augmentation policy $\mathcal{A}_w : \mathcal{X} \fromto \mathcal{X}$ and $\mathcal{A}_s : \mathcal{X} \fromto \mathcal{X}$, respectively. The weakly-augmented representation $\mathcal{A}_w(\vec{x})$ is then used to construct a pseudo-label based on $\hatp(\mathcal{A}_w(\vec{x}))$, which is finally compared to $\hatp(\mathcal{A}_s(\vec{x}))$ to compute the loss. 

Transferred to our methodology, we derive conformal predictions for the weakly-aug\-mented instance $\mathcal{A}_w(\vec{x})$ in an inductive manner based on a previously separated calibration dataset $\mathcal{D}_{\text{calib}}$. To this end, a (strongly valid) possibility distribution $\pi$ over all possible labels $y \in \mathcal{Y}$ is determined for the prediction $\hatp(\mathcal{A}_w(\vec{x}))$, which is used to construct a credal target set $\mathcal{Q}_\pi$ according to (\ref{eq:credal_set}). Finally, $\mathcal{Q}_\pi$ is compared to the prediction $\hat{p}(\mathcal{A}_s(\vec{x}))$ in terms of a generalized probabilistic loss $\mathcal{L}^* : 2^{\mathbb{P}(\mathcal{Y})} \times \mathbb{P}(\mathcal{Y}) \fromto \mathbb{R}_+$, which we shall detail in \Cref{sec:ccssl:gen_credal_labeling}.

Altogether, the training objective is given by
\begin{equation}
\begin{split}
   \mathcal{L}  &= \overbrace{\frac{1}{|\mathcal{B}_l|} \sum_{(\vec{x}_i, p_{y_i}) \in \mathcal{B}_l} \kldiv{p_{y_i}}{\phat(\vec{x}_i)}}^{\mathcal{L}_l} \, \\
   & \qquad + \, \lambda_u \underbrace{\frac{1}{|\mathcal{B}_u|} \sum_{\vec{x}_i \in \mathcal{B}_u} \mathcal{L}^*(Q_{\pi_i}, \phat((\mathcal{A}_s(\vec{x})_i))}_{\mathcal{L}_u} \enspace ,
  \end{split}
\end{equation}
where $\lambda_u \geq 0$ weights the importance of the unlabeled loss part $\mathcal{L}_u$.

\subsection{Validity Mitigates Confirmation Biases}
\label{sec:ccssl:mit_cb}

Many of the recent SSL approaches, including CSSL, leverage techniques such as consistency regularization to mitigate confirmation biases. Intuitively speaking, it enforces a smooth decision boundary in the neighborhood of (augmented) unlabeled instances, also referred to as \textit{local consistency} \citep{DBLP:conf/iclr/WeiSCM21}. For instance, classifiers trained on images showing cars provided in the training dataset predict consistent labels for similar cars that differ in the color, perspective changes or in an alternative scenery. If a certain proximity of the pseudo-labeled instances to the known labeled instances in the latent feature space is preserved, known as the \textit{expansion assumption}, CR can lead to a global class-wise consistency, which can ``de-noise'' wrong pseudo-labels. Hence, it serves as a means to alleviate the problem of misguidance through mislabels. 

However, the expansion assumption often appears too unrealistic in practice. For instance, violations happen in cases where the neighborhood of an (unlabeled) instance is not homogeneously populated by instances of the same true class. Typical examples for such situations are highly uncertain classification problems in which individual classes are hard to separate, e.g., distinguishing model variants of a car based on subtle visual differences such as badges, or data with imbalanced class frequencies. In the latter case, some classes dominate underrepresented ones, such that the neighborhood of unlabeled instances may be mostly populated by instances from different classes, thereby violating the expansion assumption. Consequently, from an empirical risk minimization point of view, it might be more reasonable to attribute larger regions populated by unlabeled instances from the minority class to a majority classes. This leaves the former overlooked, so that the correction of wrong pseudo-labels is not possible anymore. 

Credal labels as constructed in CSSL are incapable of solving the issue of the expansion violation assumption and hence ineffective CR: Similar to single-point probabilistic pseudo-labeling methods, the pseudo-labels have no error guarantees that the actual true class is adequately incorporated in the target. As opposed to this, conformal credal pseudo-labels can provide such, which is why their validity guarantee can be regarded as a second fallback when combined with CR. When the expansion assumption is violated, a higher validity of credal sets oppose committing to a misbelief in a too premature manner by a more cautious learning behavior. As will be seen in the experimental evaluation on commonly used semi-supervised image classification benchmarks, conformal (and hence valid) credal sets indeed show a more robust behavior towards confirmation biases when facing highly uncertain neighborhoods (cf. \cref{sec:exps:gen}). In addition, aiming to isolate the contributions of CR and validity in a different setting, we provide further experiments on imbalanced data in the appendix.

\subsection{Generalized Credal Learning}
\label{sec:ccssl:gen_credal_labeling}

As set-valued targets are provided as (pseudo-)supervision for the unlabeled instances, a generalization of a single-point probability loss $\mathcal{L}$ is required. Here, we follow \citep{cssl} and leverage a generalized empirical risk minimization approach by minimizing the \textit{optimistic superset loss} \citep{huellermeierchengosl,inf_loss_icml} 
\begin{equation}
    \label{eq:osl}
    \mathcal{L}^*(\mathcal{Q}_\pi, \hat{p}) := \min_{p \in \mathcal{Q}_\pi} \mathcal{L}(p, \hat{p}) \enspace ,
\end{equation}
with $\mathcal{L} = D_{KL}$ in our case. This generalization is motivated by the idea of \textit{data disambiguation} \citep{huellermeierchengosl}, in which targets are disambiguated within imprecise sets according to their plausibility in the overall loss minimization context over all data points, leading to a minimal empirical risk with respect to the original loss $\mathcal{L}$.

While credal sets as considered in CSSL are of rather simplistic nature, credal sets with arbitrary (but normalized) possibility distributions as introduced in \Cref{sec:conf_pl:conformal_credal_pl} impose a more complex optimization problem. More precisely, consider a possibility distribution $\pi$ and denote $\pi_i := \pi(y_i)$. Without loss of generality, let the possibilities be ordered, i.e., $0 \leq \pi_1 \leq \ldots, \leq \pi_K = 1$. Then, it has been shown that the following holds for any distribution $p \in \mathcal{Q}_\pi$ with $p_i := p(y_i)$  \citep{DELGADO1987311}:
\begin{equation}
    \label{eq:possibility_constraints}
    p \in \mathcal{Q}_\pi \Leftrightarrow \sum_{k=1}^i p_k \leq \pi_i \enspace , \qquad i \in \{1, ..., K\}
\end{equation}
The resulting set of inequality constraints induces a convex polytope \citep{kroupa2006many}, such that the optimization problem in (\ref{eq:osl}) becomes the problem of finding the closest point in a convex polytope (here $\mathcal{Q}_\pi$) with minimal distance to a given query point ($\hat{p}$). We provide examples illustrating such convex credal polytopes in the appendix.

Optimization problems of this kind are not new, and many approximate solutions have been proposed in the past, including projected gradient algorithms \citep{DBLP:journals/corr/abs-1107-4623}, conditional gradient descent (alias Frank-Wolfe algorithms) \citep{Frank1956AnAF,DBLP:conf/icml/Jaggi13}, and other projection-free stochastic methods \citep{DBLP:journals/access/LiCC20b}. However, we found that one can take advantage of the structure of the problem (\ref{eq:possibility_constraints}) to induce a precise and average-case efficient algorithmic solution to (\ref{eq:osl}). 

\begin{algorithm}[htbp]
    \caption{Generalized Credal Learning Loss}
    \label{alg:gen_credal_learning}
    \begin{algorithmic}
        \REQUIRE Predicted distribution $\hatp \in \mathbb{P}(\cY)$, (normalized) possibility distribution $\pi : \cY \fromto [0,1]$
        
        \STATE \algorithmicif\ $\hatp \in \mathcal{Q}_\pi$ \algorithmicthen\ \algorithmicreturn\ $\kldiv{\hatp}{\hatp} = 0$
        \STATE Initialize set of unassigned classes $Y = \cY$
        \WHILE{$Y$ is not empty}
            \STATE Determine $y^* \in Y$ with highest $\pi(y^*)$, such that the probabilities
            \begin{equation*}
                \bar{p}(y) = \left(\pi(y^*) - \sum_{y' \notin Y} p^r(y')\right) \cdot \frac{\hatp(y)}{\sum_{y' \in Y'} \hatp(y')}
            \end{equation*}
            \qquad for all $y \in Y' := \{ y \in Y \given \pi(y) \leq \pi(y^*) \}$ do not violate the constraints in \cref{eq:possibility_constraints}
            \STATE Assign $p^r(y) = \bar{p}(y)$ for all $y \in Y'$
            \STATE $Y = Y \setminus Y'$
        \ENDWHILE
        \RETURN $\kldiv{p^r}{\hatp}$
    \end{algorithmic}
\end{algorithm}

\Cref{alg:gen_credal_learning} lists the procedure to determine the loss $\mathcal{L}^*(\mathcal{Q}_\pi, \hatp)$. The key idea is to consider each face of the convex polytope being associated with a particular possibility constraint $\pi_i$. The algorithm then determines the optimal face on which the prediction $\hatp$ can be projected without violating the constraint for any $\pi_j \leq \pi_i$, which is guaranteed to provide the optimal solution for the classes $y_j$. In an iterative scheme, this procedure is applied to all classes, leading to the distribution $p \in \mathcal{Q}_\pi$ with minimal $D_{KL}$ to $\phat$. Hence, we can state the following theorem.

\begin{theorem}[Optimality]\label{theorem:optimality}
Given a credal set $\mathcal{Q}_\pi$ induced by a normalized possibility distribution $\pi : \mathcal{Y} \fromto [0,1]$ with $\max_{y \in \mathcal{Y}} \pi(y) = 1$ according to (\ref{eq:credal_set}), \Cref{alg:gen_credal_learning} returns the solution of $\mathcal{L}^*(\mathcal{Q}_\pi, \hat{p})$ as defined in (\ref{eq:osl}) for an arbitrary distribution $\hat{p} \in \mathbb{P}(\mathcal{Y})$.
\end{theorem}

Due to space limitations, we provide the proof of \cref{theorem:optimality} in the appendix. We further provide a discussion of the computational complexity of \cref{alg:gen_credal_learning} therein, showing that the worst case complexity is cubic in the number of labels\,---\,the average case complexity is much lower, however, making the method amenable to large data sets as demonstrated in our experimental evaluation.

\section{Experiments}
\label{sec:exps}

To demonstrate the effectiveness of our method, we present empirical results for the domain of image classification as an important and practically relevant application. We refer to the appendix for additional results, including a study on knowledge graphs to demonstrate generalizability of our approach, as well as a more comprehensive overview over experimental settings for reproducibility.

\subsection{Experimental Setting}
\label{sec:exps:setting}

Following previous semi-supervised learning evaluation protocols \citep{fixmatch,cssl}, we performed experiments on CIFAR-10/-100 \citep{CIFAR}, SVHN \citep{svhn} (without the extra data split) and STL-10 \citep{stl10} with various numbers of sub-selected labels. To this end, we trained Wide ResNet-28-2 \citep{wideresnet} models for CIFAR-10, SVHN and STL-10, whereas we considered Wide ResNet-28-8 as architecture for CIFAR-100. For each combination, we conducted a Bayesian optimization to tune the hyperparameters with $20$ runs each on a separate validation split, while we report the final test performances per model trained with the tuned parameters. Each model was trained for $2^{18}$ iterations. We repeated each experiment for $3$ different seeds, whereby we re-used the best hyperparameters for all seeds due to the high computational complexity.

For \ccssl, we distinguish two variants employing either the non-conformity measure (\ref{eq:non_conformity}) or (\ref{eq:non_conformity2}), which we refer to as \textit{\ccssl-diff} and \textit{\ccssl-prop} respectively.
As baselines, we consider FixMatch \citep{fixmatch} and its distribution alignment version as hard and soft probabilistic pseudo-labeling technique, respectively. Moreover, we compare our method to UDA \citep{UDA} as another soft variant. Recently, FlexMatch \citep{DBLP:conf/nips/ZhangWHWWOS21} has been proposed as an advancement of FixMatch by adding curriculum learning encompassing uncertainty-awareness. Finally, we report results of CSSL \citep{cssl} as methodically closest related work to our approach. In all cases, we employ RandAugment as strong augmentation policy to realize consistency regularization. As all of the mentioned approaches were embedded in the basic FixMatch framework, we achieve a fair comparison alleviating side-effects.

\subsection{Generalization Performance}
\label{sec:exps:gen}

\begin{table}[htbp]
    \centering
    \caption{Averaged accuracies over $3$ seeds for different numbers of labels. \textbf{Bold} entries indicate the best performing method per column.}
    \label{tab:misclassification}
    \resizebox{\textwidth}{!}{%
    \begin{tabular}{lcccccccccc}
    \toprule
    & \multicolumn{3}{c}{CIFAR-10} & \multicolumn{3}{c}{CIFAR-100} & \multicolumn{3}{c}{SVHN} & STL-10 \\
    \cmidrule(lr){2-4}
    \cmidrule(lr){5-7}
    \cmidrule(lr){8-10}
    \cmidrule(lr){11-11}
     & 40 lab. & 250 lab. & 4000 lab. & 400 lab. & 2500 lab. & 10000 lab. & 40 lab. & 250 lab. & 1000 lab. & 1000 lab. \\
    \midrule
    UDA & 86.44 {\small $\pm${2.70}} & \textbf{94.81} {\small $\pm${0.22}} & 95.31 {\small $\pm${0.10}} & 49.41 {\small $\pm${2.96}} & 68.33 {\small $\pm${0.31}} & 75.98 {\small $\pm${0.45}} & 84.82 {\small $\pm${10.6}} & 96.41 {\small $\pm${1.04}} & 97.14 {\small $\pm${0.24}} & 83.94 {\small $\pm${1.49}} \\
    FixMatch & 87.14 {\small $\pm${2.61}} & 93.81 {\small $\pm${1.02}} & 95.08 {\small $\pm${0.18}} & 47.73 {\small $\pm${1.88}} & 66.82 {\small $\pm${0.26}} & 76.66 {\small $\pm${0.18}} & 86.26 {\small $\pm${14.2}} & 96.23 {\small $\pm${1.50}} & \textbf{97.32} {\small $\pm${0.09}} & 85.34 {\small $\pm${0.92}} \\
    FixMatch DA & 89.54 {\small $\pm${5.90}} & 94.00 {\small $\pm${0.56}} & 95.13 {\small $\pm${0.21}} & 51.31 {\small $\pm${2.67}} & 69.98 {\small $\pm${0.30}} & 76.67 {\small $\pm${0.08}} & 86.00 {\small $\pm${16.3}} & 95.85 {\small $\pm${1.62}} & 97.02 {\small $\pm${0.16}} & 85.59 {\small $\pm${1.21}} \\
    FlexMatch & 91.21 {\small $\pm${3.46}} & 94.08 {\small $\pm${0.64}} & 94.62 {\small $\pm${0.27}} & 49.99 {\small $\pm${0.39}} & \textbf{71.47} {\small $\pm${1.06}} & 77.01 {\small $\pm${0.17}} & 85.78 {\small $\pm${1.37}} & 96.52 {\small $\pm${0.29}} & 96.54 {\small $\pm${0.30}} & 85.24 {\small $\pm${1.49}} \\
    \midrule
    CSSL & \textbf{91.70} {\small $\pm${4.77}} & 94.59 {\small $\pm${0.15}} & 95.41 {\small $\pm${0.04}} & 52.54 {\small $\pm${1.60}} & 67.81 {\small $\pm${0.64}} & 77.56 {\small $\pm${0.22}} & \textbf{87.27} {\small $\pm${5.69}} & 95.54 {\small $\pm${1.63}} & 96.69 {\small $\pm${0.75}} & 85.07 {\small $\pm${1.11}} \\
    \midrule
    \ccssl{}-diff & 90.13 {\small $\pm${3.33}} & 93.56 {\small $\pm${0.21}} & 95.43 {\small $\pm${0.06}} & \textbf{54.13} {\small $\pm${1.97}} & 69.33 {\small $\pm${0.78}} & 77.40 {\small $\pm${0.18}} & 86.52 {\small $\pm${6.84}} & 95.67 {\small $\pm${1.51}} & 96.97 {\small $\pm${0.23}} & 85.26 {\small $\pm${0.79}}\\
    \ccssl{}-prop & 89.24 {\small $\pm${2.56}} & 94.34 {\small $\pm${0.27}} & \textbf{95.48} {\small $\pm${0.06}} & 53.48 {\small $\pm${2.75}} & 67.90 {\small $\pm${0.37}} & \textbf{77.72} {\small $\pm${0.08}} & 85.38 {\small $\pm${7.67}} & \textbf{96.87} {\small $\pm${0.20}} & 97.04 {\small $\pm${0.38}} & \textbf{85.67} {\small $\pm${1.14}} \\
    \bottomrule
    \end{tabular}
    }
\end{table}

\Cref{tab:misclassification} shows the generalization performance of all methods with respect to the accuracy for various amounts of labeled data. Note that the calibration set for the conformal credal variant is taken from the labeled instances, i.e., the effective number of instances to learn from is further decreased. In the appendix, we provide a quantification of the network calibration, i.e., the quality of the predicted probability distributions.

As can be seen, \ccssl{} leads to competitive generalization when a sufficient amount of labeled instances is provided, often even performing best among the compared methods. In the label-scarce settings, the separation of a calibration set (which is taken from the labeled data part) has a stronger effect, and the performance is slightly inferior to CSSL. However, the performance still does not drop, which confirms the effectiveness of the improved pseudo-label quality over CSSL. In most other cases, \ccssl{} appears to be superior compared to CSSL.

On CIFAR-100 with 400 labels, the learner often faces neighborhoods populated by instances with heterogeneous class distributions. As a result, consistency regularization becomes less effective as the rich class space harms the continuity of regions covering instances of a particular class, leading to the manifestation of misbeliefs in these regions. 
As consistently observed in the learning curves for CIFAR-100 shown in Fig. \ref{fig:rebut:cifar100}, CSSL with CR suffers from this issue with increased confidence.
In the terminal phase of the training, credal sets are kept rather vague, which is why misbeliefs do not affect the overall performance too much. However, later phases show a performance degradation, which can be attributed to the fact that (potentially mislabeled) credal pseudo-labels become smaller and have thus a higher weight in the overall loss minimization. 
As opposed to that, \ccssl{} learns more cautiously thanks to the conformal construction of credal sets and can continuously improve the generalization performance.

\begin{figure}[htbp]
    \centering
    \includegraphics[trim={0cm 0.35cm 0cm 0.25cm},clip,width=0.6\columnwidth]{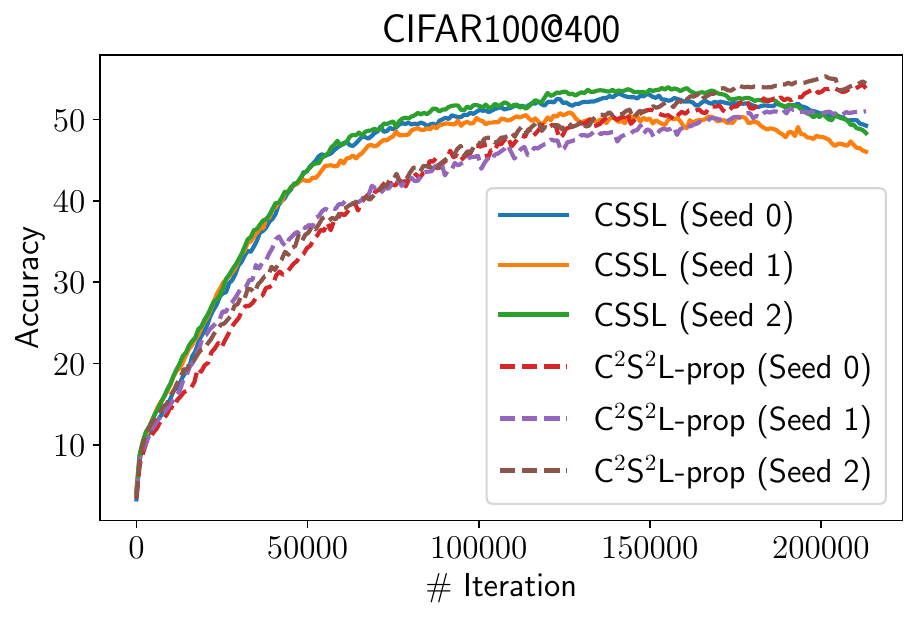}
    \caption{Test accuracies over the course of the training on CIFAR-100 with 400 labels.}
    \label{fig:rebut:cifar100}
\end{figure}

\subsection{Pseudo-Label Quality}
\label{sec:exps:pl_quality}

To assess the quality of the (credal) pseudo-labels, we report the validity according to (\ref{eq:strong_validity}) by measuring the error rate in terms of $\mathbbm{1}(\pi(y) \leq \delta)$
given the pseudo-supervisions $\pi$ for the unlabeled training instances with true labels $y$ (for significance levels $\delta \in \{0.05, 0.1, 0.25\}$). To ensure fairness, we compute these scores for trained models so as not to give an advantage to \ccssl, which can in contrast to CSSL rely on the strong validity guarantee throughout the training. 

\begin{table}[htbp]
    \centering
    \caption{The final validity as specified in \cref{eq:strong_validity} of all credal pseudo-labels for different significance levels $\delta$ averaged over $3$ random seeds. \textbf{Bold} entries indicate the best method per column, the standard deviation is a factor of $1e^{-3}$.}
    \label{tab:validity}
    \resizebox{\textwidth}{!}{%
    \begin{tabular}{lcccccccccccc}
    \toprule
    & \multicolumn{6}{c}{CIFAR-10} & \multicolumn{6}{c}{SVHN} \\
    \cmidrule(lr){2-7}
    \cmidrule(lr){8-13}
     & \multicolumn{3}{c}{250 lab.} & \multicolumn{3}{c}{1000 lab.} & \multicolumn{3}{c}{250 lab.} & \multicolumn{3}{c}{1000 lab.} \\
    \cmidrule(lr){2-4}
    \cmidrule(lr){5-7}
    \cmidrule(lr){8-10}
    \cmidrule(lr){11-13}
    $\delta$ & $0.05$ & $0.1$ & $0.25$ & $0.05$ & $0.1$ & $0.25$ & $0.05$ & $0.1$ & $0.25$ & $0.05$ & $0.1$ & $0.25$ \\
    \midrule
    CSSL & 0.027 {\small $\pm${0.6}} & 0.033 {\small $\pm${1.1}} & \textbf{0.042} {\small $\pm${0.7}} & 0.033 {\small $\pm${1.1}} & 0.037 {\small $\pm${4.8}} & 0.042 {\small $\pm${1.1}} & 0.038 {\small $\pm${0.6}} & 0.044 {\small $\pm${1.4}} & 0.053 {\small $\pm${0.8}} & 0.030  {\small $\pm${1.6}} & 0.034 {\small $\pm${6.1}} & 0.040 {\small $\pm${2.1}} \\
    \midrule
    \ccssl{}-diff & 0.026 {\small $\pm${2.9}} & 0.032 {\small $\pm${1.9}} & 0.047 {\small $\pm${4.7}} & 0.029 {\small $\pm${1.5}} & 0.034 {\small $\pm${1.5}} & 0.042 {\small $\pm${0.6}} & 0.028 {\small $\pm${1.0}} & 0.035 {\small $\pm${1.1}} & 0.042 {\small $\pm${1.3}} & 0.024 {\small $\pm${1.2}} & 0.028 {\small $\pm${1.6}} & 0.032 {\small $\pm${1.2}}\\
    \ccssl{}-prop & \textbf{0.022} {\small $\pm${2.1}} & \textbf{0.031} {\small $\pm${1.5}} & 0.043 {\small $\pm${1.7}} & \textbf{0.024} {\small $\pm${2.7}} & \textbf{0.029} {\small $\pm${3.4}} & \textbf{0.039} {\small $\pm${1.9}} & \textbf{0.014} {\small $\pm${1.1}} & \textbf{0.020} {\small $\pm${1.1}} & \textbf{0.024} {\small $\pm${0.3}} & \textbf{0.021} {\small $\pm${1.1}} & \textbf{0.023} {\small $\pm${2.7}} & \textbf{0.029} {\small $\pm${1.5}} \\
    \bottomrule
    \end{tabular}}
\end{table}

In the context of self-training, a lower error rate is desirable to obtain less noisy and more informative self-supervision. As shown in \Cref{tab:validity}, the two variants of \ccssl{} indeed achieve a consistent improvement in the validity of the pseudo-labels over CSSL in terms of the error rate, although the supervision provided by the latter is already quite accurate and does not violate the strong validity property for these instances either. Moreover, \ccssl-prop leads to more accurate sets compared to \ccssl-diff. 
Note that in standard conformal prediction one is typically interested in error rates that match the respective significance levels. In the setting considered here, this is undermined by relatively small calibration sets $\mathcal{D}_{\text{calib}}$, as well as the fact that the unlabeled instances, whose validity is presented here, have already been observed in previous training iterations, thus leading to optimistic error rates. Nevertheless, this optimism turns out to be beneficial for effective self-training, as our empirical results confirm. 

\begin{figure}[htbp]
    \centering
    \includegraphics[trim={0cm 0.35cm 0cm 0.25cm},clip,width=1.0\textwidth]{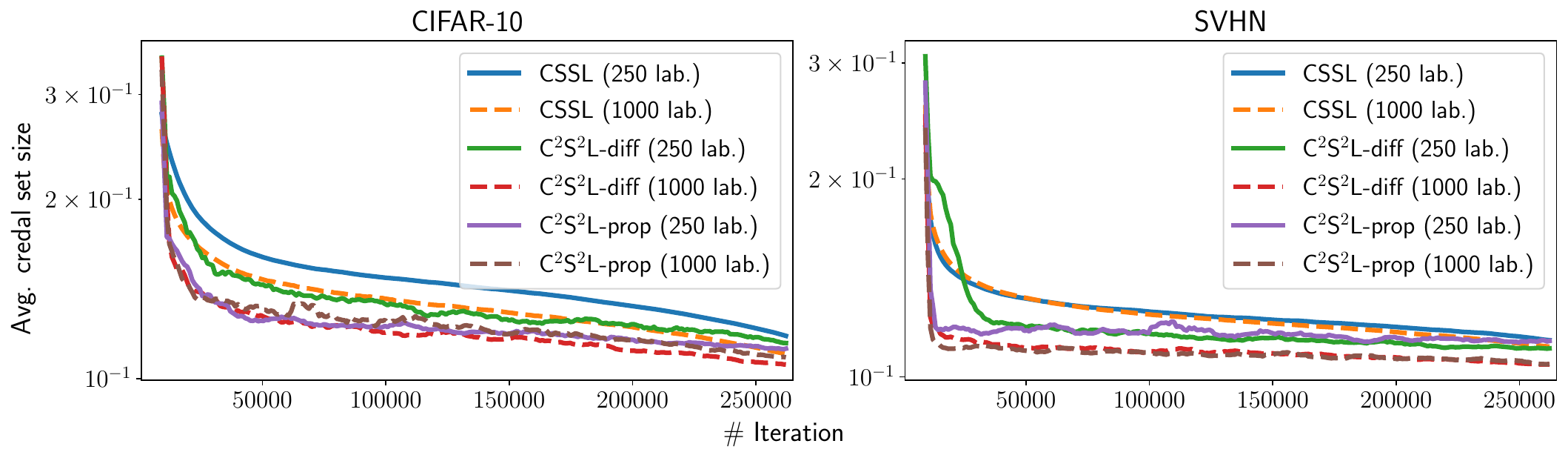}
    \caption{The credal set efficiency in terms of the mean possibilities $\pi$ for each class, averaged over all pseudo-labels in each iteration for all $3$ seeds.}
    \label{fig:efficiency}
\end{figure}

Furthermore, the efficiency of the credal sets, i.e., their sizes, has an effect on the learning behavior. \cref{fig:efficiency} compares the credal set sizes of CSSL and \ccssl{} over the course of the training. As can be seen, \ccssl{} constructs smaller credal sets, which in combination with the improved validity leads to a more effective supervision. Moreover, the credal sets sizes decrease with higher numbers of labels, which is due to a higher prediction quality involved in the credal set construction. Apart from the results on CIFAR-10 with 250 labels, one can further see that the credal set size deviates only slightly between the two non-conformity measures employed in \ccssl.

\section{Conclusion}
\label{sec:conclusion}

In the context of semi-supervised learning, previous pseudo-labeling approaches lack validity guarantees for the quality of the pseudo-supervision. Such methods often suffer from misleading supervision, leaving much potential unused. In our work, we address these shortcomings by a conformal credal labeling approach leveraging the framework of conformal prediction, which entails validity guarantees for the constructed credal pseudo-labels. Our empirical study confirms the adequacy of this approach when combined with consistency regularization in terms of generalization performance, as well as the calibration and efficiency of pseudo-labels compared to previous methods. At the same time, the combination of a rigorously studied uncertainty quantification framework with pseudo-labeling paves the way for more thorough theoretical analyses in the field of self-training in semi-supervised learning.

In future work, we plan to investigate approaches to achieve \textit{conditional} (per-instance) validity, i.e., conformal credal labels specifically tailored to (conditioned on) the query instance $\vec{x}_{N+1}$, which would lead to even stronger guarantees for the quality of pseudo-labels. Approximations of this type of validity for conformal prediction have already been provided \citep{DBLP:journals/ml/Vovk13,DBLP:conf/copa/Bellotti21}. Moreover, a more rigorous analysis of the efficiency of alternative non-conformity scores needs to be performed, including their robustness to label noise as typically present in real-world data.

\acks{This work was partially supported by the German Research Foundation (DFG) within the Collaborative Research Center ``On-The-Fly Computing'' (CRC 901 project no.~160364472). Moreover, the authors gratefully acknowledge the funding of this project by computing time provided by the Paderborn Center for Parallel Computing (PC$^2$).}

\bibliography{egbib}

\newpage

\appendix

\section{Pseudo-Code of \ccssl}

\begin{algorithm}[htbp]
    \caption{\ccssl{} with consistency regularization}
    \label{alg:ccssl}
    \begin{algorithmic}[1]
        \REQUIRE Batch of labeled instances with degenerate ground truth distributions $\mathcal{B}_l = \{(\vec{x}_i, p_i)\}_{i=1}^B \in \left(\mathcal{X} \times \mathcal{Y}\right)^B$, unlabeled batch ratio $\mu$, batch $\mathcal{B}_u = \{ \vec{x}_i \}_{i=1}^{\mu B}$ of unlabeled instances, unlabeled loss weight $\lambda_u$, model $\hat{p} : \mathcal{X} \to \mathbb{P}(\mathcal{Y})$, strong and weak augmentation functions $\mathcal{A}_s, \mathcal{A}_w : \mathcal{X} \to \mathcal{X}$, calibration data $\mathcal{D}_{\text{calib}} \subset (\mathcal{X} \times \mathcal{Y})^L$, inductive conformal prediction procedure $ICP : (\mathcal{X} \times \mathcal{Y})^L \times \mathbb{P}(\mathcal{Y}) \to (\mathcal{Y} \to [0,1])$
        \STATE $\mathcal{L}_l = \frac{1}{B} \sum_{(\vec{x}, p) \in \mathcal{B}_l} \kldiv{p}{ \hat{p}(\mathcal{A}_w(\vec{x}))}$
        \STATE Initialize pseudo-labeled batch $\mathcal{U} = \emptyset$
        \FOR{all $\vec{x} \in \mathcal{B}_u$}
            \STATE Derive possibility distribution $\pi = ICP(\mathcal{D}_{\text{calib}}, \hatp(\mathcal{A}_w(\vec{x})))$
            \STATE Apply normalization to $\pi$ such that $\max_{y\in \mathcal{Y}} \pi(y) = 1$ (e.g., as in \cref{eq:norm1})
            \STATE Construct credal set $\mathcal{Q}_\pi$ as in \cref{eq:credal_set}
            \STATE $\mathcal{U} = \mathcal{U} \cup \{(\vec{x}, Q_\pi)\}$
        \ENDFOR
        \STATE $\mathcal{L}_u = \frac{1}{\mu B} \sum_{(\vec{x}, Q_\pi) \in \mathcal{U}} \mathcal{L}^*(Q_\pi, \hat{p}(\mathcal{A}_s(\vec{x})))$ \COMMENT{Can be solved via generalized credal learning procedure (Alg.\ 1 in main paper)}
        
        \RETURN $\mathcal{L}_l + \lambda_u \mathcal{L}_u$
    \end{algorithmic}
\end{algorithm}

\section{Experimental Details}

\subsection{Settings}

To conduct the experiments as presented in the paper, we followed the basic semi-supervised learning evaluation scheme as in \citep{fixmatch,cssl}. However, as opposed to previous evaluations, we reduce the number of iterations to $2^{18}$ with a batch size of $32$, which allows for a proper hyperparamter optimization of all methods. To this end, we employ a Bayesian optimization\footnote{We used the Bayesian optimization implementation as offered by \textit{Weights \& Biases} \citep{wandb} with default parameters.} with $20$ runs for each combination of dataset and number of labels on a separate validation split. Moreover, we use Hyperband \citep{DBLP:journals/jmlr/LiJDRT17} with $\eta = 3$ and $20$ minimum epochs (that is, iterating over all unlabeled instances once) for early stopping.
Due to the computational complexity of this procedure, we determined the best hyperparameter on a fixed seed and applied those parameters to all repetitions with different seeds for the same dataset and number of labels combination. Albeit not being ideal, such routine still improves fairness compared to previous evaluations which do not apply the same hyperparameter tuning procedure to all regarded baselines.

\begin{table}[htbp]
    \caption{Hyperparameter search spaces considered in the optimization.}
    \centering
    \begin{tabular}{lll}
    \toprule
    Method & Parameter & Values \\
    \midrule
    \multirow{4}{*}{All} & Initial learning rate & $\{0.005, 0.01, 0.03, 0.05, 0.1\}$ \\
    & Unlabeled batch multiplicity $\mu$ & $\{ 3, 7 \}$ \\
    & Weight decay & $\{ 0.0005, 0.0001 \}$ \\
    & Unlabeled loss weight $\lambda_u$ & $\{ 1 \}$ \\
    \midrule
    \multirow{2}{*}{FixMatch (DA), UDA} & Confidence threshold $\tau$ & $\{0.7, 0.8, 0.9, 0.95 \}$ \\
    & Temperature & $\{0.5, 1\}$ \\
    \midrule
    \multirow{2}{*}{FlexMatch} & Cutoff threshold & $\{ 0.8, 0.9, 0.95 \}$ \\
    & Threshold warmup & $\{ \text{True}, \text{False} \}$ \\
    \midrule
    \ccssl-diff & Calibration split & $\{0.1, 0.25, 0.5\}$ \\
    \midrule
    \multirow{2}{*}{\ccssl-prop} & Calibration split & $\{0.1, 0.25, 0.5\}$ \\
    & Non-conf. sensitivity $\gamma$ & $\{ 0.01, 0.1, 1 \}$ \\
    \bottomrule
    \end{tabular}
    \label{table:hyperparameters}
\end{table}

\cref{table:hyperparameters} shows the considered parameter spaces. To train the models, we use SGD with a Nesterov momentum of $0.9$. We further employ cosine annealing as learning rate schedule \citep{DBLP:conf/iclr/LoshchilovH17}. Moreover, we apply exponential moving averaging with a fixed decay of $0.999$ to the weights.

\subsection{Code and Environment}

Our official implementation is publicly available.\footnote{\url{https://github.com/julilien/C2S2L}} Therein, we implemented all methods using PyTorch\footnote{\url{https://pytorch.org/}, BSD-style license}, where we reused the official implementations if available. We proceeded from a popular FixMatch re-implementation in PyTorch\footnote{\url{https://github.com/kekmodel/FixMatch-pytorch}, MIT license} for the image classification experiments, which we carefully checked for any differences to the original repository, and embedded all other baselines into it. To conduct the experiments, we used several Nvidia A100 GPUs in a modern high performance cluster environment.

\section{Additional Results}

\subsection{Predictor Calibration}

For completeness, we present the quality of the prediction probability distributions in the large-scale image classification experiments with respect to their expected calibration errors in \cref{tab:ece}. Both \ccssl{} variants demonstrate favorable calibration properties, whereas \ccssl-prop often outperforms all other methods.

\begin{table}[htbp]
    \centering
    \caption{Averaged ECE scores with $15$ bins over $3$ seeds for different numbers of labels. \textbf{Bold} entries indicate the best performing method per column. The standard deviation is a factor of $1e^{-2}$.}
    \label{tab:ece}
    \resizebox{\columnwidth}{!}{%
    \begin{tabular}{lcccccccccc}
    \toprule
    & \multicolumn{3}{c}{CIFAR-10} & \multicolumn{3}{c}{CIFAR-100} & \multicolumn{3}{c}{SVHN} & STL-10 \\
    \cmidrule(lr){2-4}
    \cmidrule(lr){5-7}
    \cmidrule(lr){8-10}
    \cmidrule(lr){11-11}
     & 40 lab. & 250 lab. & 4000 lab. & 400 lab. & 2500 lab. & 10000 lab. & 40 lab. & 250 lab. & 1000 lab. & 1000 lab. \\
    \midrule
    UDA & 0.159 {\small $\pm${7.9}} & \textbf{0.051} {\small $\pm${0.2}} & 0.046 {\small $\pm${0.1}} & 0.420 {\small $\pm${2.6}} & 0.232 {\small $\pm${0.5}} & 0.173 {\small $\pm${0.5}} & 0.124 {\small $\pm${1.1}} & 0.037 {\small $\pm${1.3}} & 0.030 {\small $\pm${0.1}} & 0.140 {\small $\pm${0.9}} \\
    FixMatch & 0.136 {\small $\pm${4.3}} & 0.059 {\small $\pm${1.0}} & 0.048 {\small $\pm${0.1}} & 0.417 {\small $\pm${2.3}} & 0.254 {\small $\pm${0.5}} & 0.168 {\small $\pm${0.1}} & 0.275 {\small $\pm${34.9}} & 0.038 {\small $\pm${1.3}} & \textbf{0.027} {\small $\pm${0.1}} & 0.128 {\small $\pm${0.8}} \\
    FixMatch DA & 0.131 {\small $\pm${11.1}} & 0.057 {\small $\pm${0.6}} & 0.047 {\small $\pm${0.1}} & \textbf{0.347} {\small $\pm${2.4}} & 0.234 {\small $\pm${0.4}} & 0.169 {\small $\pm${0.1}} & \textbf{0.101} {\small $\pm${3.1}} & 0.041 {\small $\pm${1.5}} & 0.029 {\small $\pm${0.1}} & 0.127 {\small $\pm${0.8}} \\
    FlexMatch & 0.114 {\small $\pm${3.7}} & 0.057 {\small $\pm${0.5}} & 0.052 {\small $\pm${0.2}} & 0.404 {\small $\pm${0.3}} & 0.232 {\small $\pm${0.5}} & 0.169 {\small $\pm${0.2}} & 0.126 {\small $\pm${1.3}} & 0.039 {\small $\pm${0.3}} & 0.033 {\small $\pm${0.4}} & 0.130 {\small $\pm${0.9}} \\
    \midrule
    CSSL & \textbf{0.079} {\small $\pm${4.4}} & 0.053 {\small $\pm${0.1}} & 0.045 {\small $\pm${0.0}} & 0.368 {\small $\pm${0.8}} & 0.233 {\small $\pm${0.6}} & 0.167 {\small $\pm${0.2}} & 0.121 {\small $\pm${5.1}} & 0.041 {\small $\pm${1.1}} & 0.032 {\small $\pm${0.1}} & 0.126 {\small $\pm${0.9}} \\
    \midrule
    \ccssl{}-diff & 0.139 {\small $\pm${3.2}} & 0.058 {\small $\pm${0.1}} & 0.045 {\small $\pm${0.1}} & 0.352 {\small $\pm${1.6}} & \textbf{0.222} {\small $\pm${0.3}} & 0.165 {\small $\pm${0.1}} & 0.117 {\small $\pm${7.2}} & 0.044 {\small $\pm${1.2}} & 0.030 {\small $\pm${0.3}} & 0.129 {\small $\pm${0.6}}\\
    \ccssl{}-prop & 0.101 {\small $\pm${2.2}} & 0.054 {\small $\pm${0.2}} & \textbf{0.044} {\small $\pm${0.1}} & \textbf{0.347} {\small $\pm${2.3}} & 0.227 {\small $\pm${0.2}} & \textbf{0.162} {\small $\pm${0.2}} & 0.128 {\small $\pm${8.6}} & \textbf{0.033} {\small $\pm${0.1}} & 0.028 {\small $\pm${0.2}} & \textbf{0.124} {\small $\pm${0.8}} \\
    \bottomrule
    \end{tabular}
    }
\end{table}

\subsection{Learning Curves}

\begin{figure}[htbp]
    \centering
    \includegraphics[width=\textwidth]{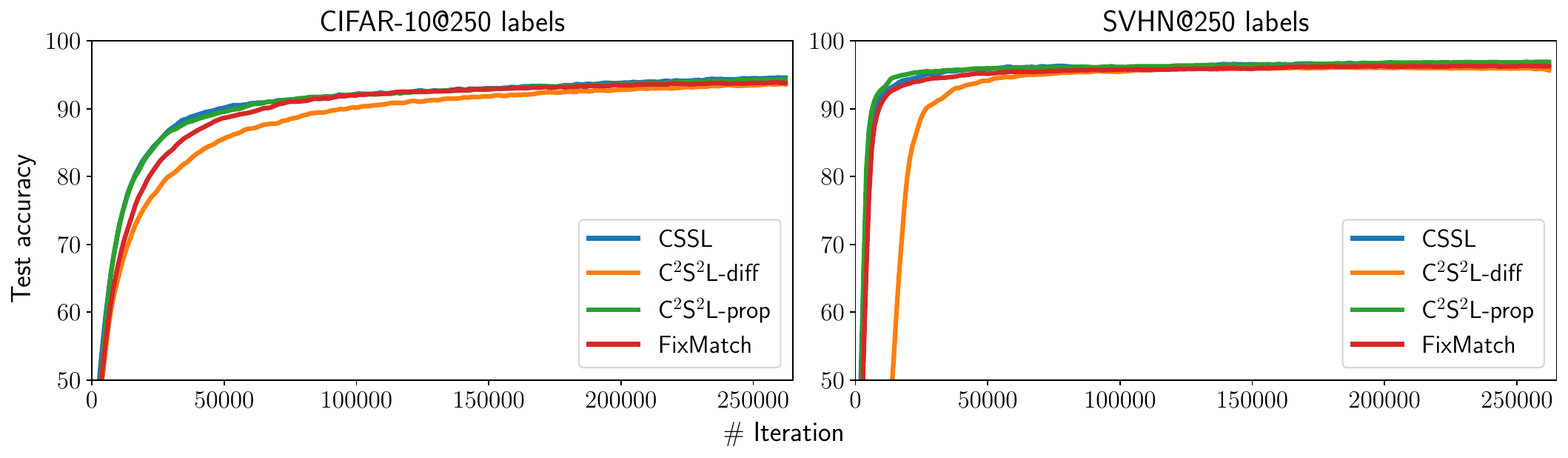}
    \caption{Test accuracies over the course of the training. The results are averaged over $3$ different random seeds.}
    \label{fig:abl:learning_curves}
\end{figure}

In addition, we provide the learning curves in terms of test accuracy per training iteration averaged over $3$ random seeds on CIFAR-10 and SVHN with 250 labels each in \cref{fig:abl:learning_curves}. As can be seen, \ccssl-prop shows a similar training efficiency as CSSL, whereas \ccssl-diff realizes a more cautious learning than all other methods. This becomes particularly visible for SVHN. Here, \ccssl-prop shows an even improved training efficiency compared to CSSL. Remarkably, the \ccssl{} variants have less labeled supervision available as part of it is separated in form of the calibration split, which can be one reason for the cautiousness of \ccssl{}-diff in the first iterations. Together with the validity gains in the pseudo-label quality, especially \ccssl-prop demonstrates its effectiveness to the task of semi-supervised learning. 

\subsection{Mitigation of Confirmation Biases: Imbalanced Data}

As discussed in the main paper, consistency regularization (CR) and the validity of conformal (credal) pseudo-labels serve as means to tackle confirmation biases. In addition to the previously shown experiments, we consider here another experimental setting aiming to isolate their individual contributions to the mitigation of confirmation biases. Namely, we look at EMNIST-ByClass \citep{DBLP:journals/corr/CohenATS17} consisting of 814,255 handwritten letters (both lower and upper case) and digits, constituting 62 imbalanced classes in total.\footnote{We refer to Fig.\ 2 in \citep{DBLP:journals/corr/CohenATS17} for a detailed overview over the class distribution.} The class imbalance leads to an attenuation of CR as frequently occurring classes dominate underrepresented ones. More technically, the neighborhood of instances belonging to an underrepresented region may be mostly populated by instances from different classes, thereby violating the expansion assumption \citep{DBLP:conf/iclr/WeiSCM21}. Consequently, from an empirical risk minimization point of view, it might be more reasonable to attribute larger regions populated by unlabeled instances from the minority class to a majority classes. This leaves the former overlooked, so that a ``de-noising'' of the pseudo-labels is not possible anymore. Again, the validity guarantees provided by the conformal prediction framework serve as a fallback here.

For this dataset, we consider either 250 or 500 labeled instances and train for $2^{15}$ iterations, keeping all other experimental parameters the same as before. Adopting the reported optimal hyperparameters for CSSL in \citep{cssl}, we used a fixed learning rate of $0.03$, SGD with Nesterov momentum of $0.9$ and trained a Wide ResNet-28-2 with a batch size of $32$ for three different seeds. Table \ref{table:conf_bias:emnist} shows the resulting generalization performances for the individual methods. Moreover, Fig. \ref{fig:rebut:emnist} presents the learning curves, which show similar but even more extreme trends as also observed in the CIFAR-100 experiments presented in Sec.\ 5.2 of the main paper.

\begin{table}[htbp]
    \caption{Test accuracies on EMNIST-ByClass. The presented results and their standard deviations are computed over $3$ seeds.}
    \centering
    \begin{tabular}{lcc}
      \toprule
        & 250 lab. & 500 lab. \\
        \midrule
    CSSL & 49.96 {\small $\pm${4.08}} & 62.74 {\small $\pm${2.15}} \\
      \midrule 
      \ccssl-diff & \textbf{59.22} {\small $\pm${3.93}} & 67.39 {\small $\pm${3.74}}   \\
      \ccssl-prop & 57.90 {\small $\pm${5.85}} & \textbf{67.62} {\small $\pm${4.61}}  \\
    \bottomrule
\label{table:conf_bias:emnist}
\end{tabular}
\end{table}

\begin{figure}[htbp]
    \centering
    \begin{minipage}{.49\textwidth}
        \centering
        \includegraphics[width=\linewidth]{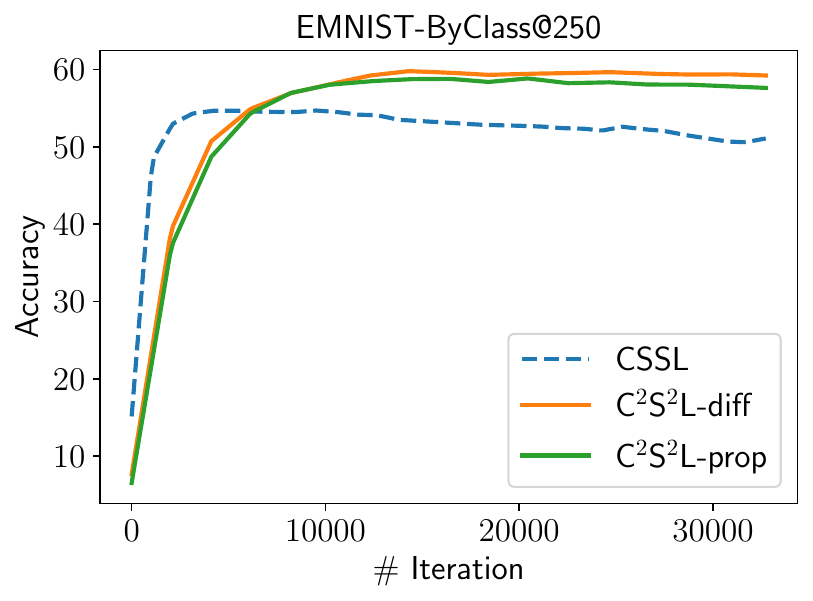}
    \end{minipage}
    \begin{minipage}{.49\textwidth}
        \centering
        \includegraphics[width=\linewidth]{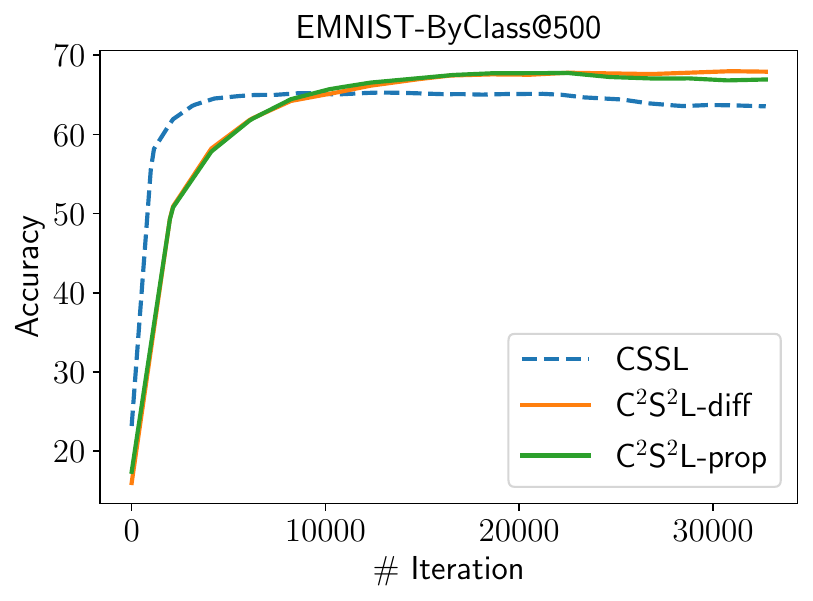}
    \end{minipage}
    \caption{Averaged learning curves in terms of test accuracies over three seeds on EMNIST-ByClass with 250 and 500 labels.}
    \label{fig:rebut:emnist}
\end{figure}

\subsection{Ablation Studies}

In the following, we present further ablation studies to investigate properties of \ccssl{} more thoroughly. If not stated otherwise, we set the initial learning rate to $0.03$, $\lambda_u=1$, $\mu=7$ and the weight decay to $0.0005$. Also, we consider calibration size fractions of $0.25$ by default.

\paragraph{Possibility Distribution Normalization} 

Conformal Credal Pseudo-Labeling involves normalizing possibility distributions $\pi : \mathcal{Y} \fromto [0,1]$ to satisfy $\max_{y \in \mathcal{Y}} \pi(y) = 1$. In \cref{eq:norm1} of the paper, we introduced a proportion-based normalization technique, to which we refer as \textit{normalization 1}. As an alternative, one can consider the following normalization \citep{DBLP:conf/isipta/CellaM21}, to which we refer as \textit{normalization 2}:

\begin{equation}
    \label{eq:norm2}
    \pi(\hat{y}) = \left\{ \begin{array}{cl}
        1 & \text{if } \hat{y} = \argmax_{y' \in \mathcal{Y}} \pi(y'), \\
        \pi(\hat{y}) & \text{otherwise.}
    \end{array} \right.
\end{equation}

It is easy to see that (\ref{eq:norm2}) leads to smaller credal sets due to $0 \leq \pi(\cdot) \leq 1$.

\begin{table}[htbp]
    \caption{Test accuracies per normalization and \ccssl{} variant for $250$ labels each. The presented results and their standard deviations are computed over $3$ seeds.}
    \centering
    \begin{tabular}{lcccc}
      \toprule
        &\multicolumn{2}{c}{CIFAR-10}    & \multicolumn{2}{c}{SVHN}\\
         \cmidrule(lr){2-3}  \cmidrule(lr){4-5}                       
         & Norm. 1 & Norm. 2 &  Norm. 1 & Norm. 2 \\
      \midrule 
      \ccssl-diff & 92.38 {\small $\pm${1.14}} & 92.71 {\small $\pm${0.81}} & 95.93 {\small $\pm${1.78}} & 95.70 {\small $\pm${1.81}}   \\
      \ccssl-prop & 92.26 {\small $\pm${2.30}} & 91.84 {\small $\pm${1.13}} & 96.61 {\small $\pm${1.08}} & 95.72 {\small $\pm${1.50}}   \\
    \bottomrule
\label{table:abl:normalization}
\end{tabular}
\end{table}

\cref{table:abl:normalization} shows the results. Although the second normalization strategy leads to smaller credal sets, it leads to inferior generalization performance for the proportion-based non-conformity measure, suggesting that an overly extreme credal set construction may be suboptimal. This supports the adequacy of the first normalization strategy as employed in \ccssl{} by default.

\begin{figure}[htbp]
    \centering
    \includegraphics[width=\textwidth]{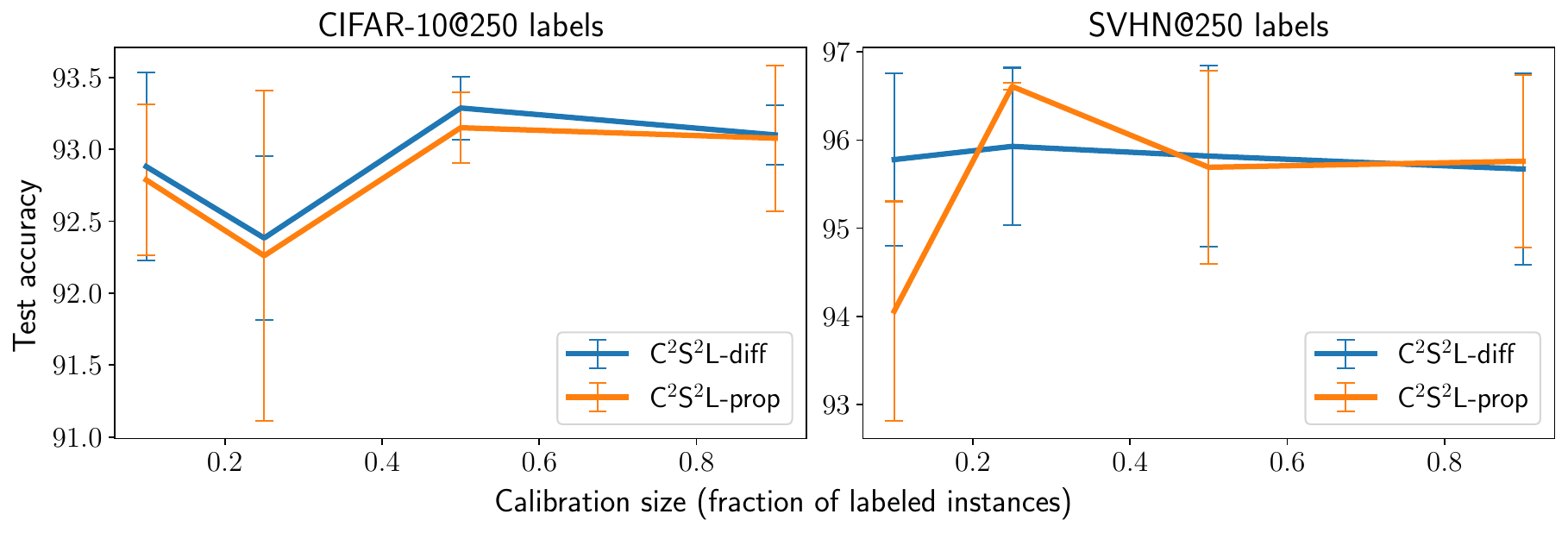}
    \caption{Test accuracy and the standard deviation per calibration size for the two variants of \ccssl. The results are averaged over $3$ different random seeds.}
    \label{fig:abl:calibration}
\end{figure}

\paragraph{Calibration Split Size} The calibration size used to determine the number of calibration instances affects the quality of the credal sets. Here, we consider calibration split proportions in $\{0.1, 0.25, 0.5, 0.9\}$. As can be seen in \cref{fig:abl:calibration}, the differences in the performances do not vary too much. On CIFAR-10 with 250 labels, no clear trend can be observed. However, the deviations of runs with a fraction of $0.25$ appear significantly higher than for the other sizes. This could be due to the sacrifice of too many labeled instances without achieving too precise pseudo-supervision. Lower and higher calibration sizes may overcome this by benefiting from either of these two extremes (higher pseudo-label quality through many labeled instances or larger calibration sets). In case of SVHN, \ccssl-prop seems to be more sensitive to the calibration size and achieves the best results with a calibration size fraction of $0.25$. Here, too small calibration sizes clearly lead to unsatisfying results, demonstrating again the influence of the pseudo-label quality on the overall generalization performance.

\paragraph{Proportion-based Non-Conformity Sensitivity} As defined in \cref{eq:non_conformity2}, the proportion-based non-conformity measure $\alpha(\mathcal{D}, (\vec{x}, y))$ involves the parameter $\gamma \geq 0$, which represents the sensitivity towards the influence of the prediction $\hatp_\mathcal{D}(\vec{x})(y)$ on the scoring for a given tuple $(\vec{x}, y)$.

\begin{table}[htbp]
    \caption{Averaged test accuracies and their standard deviations over $3$ random seeds for various $\gamma$ values used in \ccssl-prop. For each dataset, we considered $250$ labeled examples to be given.}
    \centering
    \begin{tabular}{lcc}
      \toprule
       $\gamma$ & CIFAR-10    & SVHN\\
      \midrule 
      $0.01$ & \textbf{93.64} {\small $\pm${0.33}} & 95.60 {\small $\pm${2.17}} \\
      $0.1$ & 93.24 {\small $\pm${1.16}} & 95.91 {\small $\pm${1.76}} \\
      $0.5$ & 93.19 {\small $\pm${1.02}} & \textbf{96.82} {\small $\pm${0.10}} \\
      $1$ & 93.01 {\small $\pm${1.23}} & 96.61 {\small $\pm${0.08}} \\
      $10$ & 92.29 {\small $\pm${1.94}} & 95.41 {\small $\pm${2.07}} \\
    \bottomrule
\label{table:abl:sensitivity}
\end{tabular}
\end{table}

In \cref{table:abl:sensitivity}, we report results of \ccssl-prop for $\gamma \in \{0.01, 0.1, 0.5, 1, 10\}$. While smaller $\gamma$ values lead to better results for CIFAR-10, SVHN benefits from slightly higher $\gamma$ values. These results show that this parameter indeed has an effect on the overall results, i.e., it is reasonable to consider it as a hyperparameter.

\subsection{Knowledge Graph Embedding Experiments}
We were interested in evaluating conformal credal self-supervised learning in the link prediction problem on knowledge graphs. Knowledge graphs represent structured collections of facts~\citep{hogan2020knowledge}, which are stored in graphs connecting entities via relations.
These collections of facts have been used in a wide range of applications, including web search, question answering, and recommender systems~\citep{nickel2015review}. 
The task of identifying missing links in knowledge graphs is referred to as~\textit{link prediction}. 
Knowledge graph embedding (KGE) models have been particularly successful at tackling the link prediction task, among many others~\citep{nickel2015review}. In semi-supervised link prediction, only a fraction of facts is given. The task is then to ``enrich'' the input graph to detect relations that connect entities, which are subsequently also used in the learning of graph embeddings.

\paragraph{Experimental Setting} 
In our experiments, we follow a standard training and evaluation setup commonly used in the KGE domain~\citep{ruffinelli2020you,zhang2019quaternion}. 
We consider three multiplicative interaction-based KGE embedding models: DistMult~\citep{yang2014embedding}, ComplEx~\citep{trouillon2016complex}, and QMult~\citep{pmlr-v157-demir21a}. 
To train the models, we model the problem as a 1vsAll classification (see \citep{https://doi.org/10.48550/arxiv.2205.06560} and \citep{ruffinelli2020you} for more details about the 1vsAll training regime). Here, we compare conventional pseudo-labeling as described in \citep{lee2013pseudo} to \ccssl-diff. In the following, we refer to the former as PL. We do not employ consistency regularization or any other confirmation bias mitigation technique, demonstrating the flexibility of our framework. 
In our experiments, we used the two benchmark datasets UMLS and KINSHIP~\citep{dettmers2018convolutional}, whose characteristics are provided in \Cref{table:kge_dataset}. 

\begin{table}[htbp]
    \caption{Overview of the datasets considered in the KGE experiments.}
    \centering
    \begin{tabular}{l c c c c c}
    \toprule
     & \# Entities &  \# Relations & $|\mathcal{D}_{\text{train}}|$ & $|\mathcal{D}_{\text{val}}|$ &  $|\mathcal{D}_{\text{test}}|$\\
    \midrule
    UMLS             &  136    &     93  &    10,432  &  1,304 &  1,965 \\
    KINSHIP          &  105    &     51  &    17,088  &  2,136 &  3,210 \\
    \bottomrule
    \end{tabular}
    \label{table:kge_dataset}
\end{table}

For each model (DistMult, ComplEx, and QMult), we applied a grid-search over the learning rates $\{0.01, 0.1, 0.001 \}$, batch sizes $\{512, 1024\}$ and the number of epochs $\{5, 50\}$ to tune the parameters on a separate validation set.
For \ccssl-diff, we initially divided $\mathcal{D}_{\text{train}}$ into training, calibration and unlabeled splits with 40:40:20 ratios, respectively.
For the conventional pseudo-labeling procedure, we divided $\mathcal{D}_{\text{train}}$ into training, and unlabeled splits with 40:60 ratios, respectively.
To ensure that each model is trained with the exact training split, we select the first 40\% of all triples as training split. 

As opposed to domains like image classification, where a small number of classes, for which labeled instances are provided, allows for a certain degree of interpolation, knowledge graphs involve a typically larger vocabulary of entities. By subselecting data from the knowledge graph, there is a much higher chance to miss some entities. Not observing parts of the vocabulary renders the task of learning their embeddings effectively as an unsupervised learning problem. This is why our considered training splits are relatively large. Arguably, \ccssl{} gets an unfair advantage here by providing (labeled) calibration data beyond the labeled training data. However, this data is excluded from being used as pseudo-labeled training data either, and can not contribute to the learned embeddings directly. It leaves large parts of the knowledge graph unconnected as it prevents to enrich that part of the graph by pseudo-labels, having an influence on the generalizability of the learned KGE model.

\paragraph{Link Prediction Results}
Table~\ref{table:umls_kinship_1vsall} reports the link prediction performance on UMLS and KINSHIP. 
Overall, the results suggest that incorporating \ccssl~in knowledge graph embedding model leads to better generalization performance in 16 out of 18 metrics on two benchmark datasets. Also, the credal pseudo-label construction is effective as it improves the results over the PL baseline.
To address our discussions about fairness in the split modeling, we conduct two more experiments with 40:10:50 and 30:20:50 split ratios to quantify the impact of the calibration signal in the link prediction task, where the PL baselines observes splits with ratios 40:60 and 50:50, respectively.

\begin{table}[htbp]
    \caption{Link prediction results on UMLS and KINSHIP. 
    The 40:40:20 split ratio for \ccssl-diff and a 40:60 for conventional pseudo-labeling (PL). Bold entries denote best results per method, dataset and metric.}
    \centering
    \begin{tabular}{l  c c c c c c c c c c c c c c c}
      \toprule
                           &\multicolumn{3}{c}{\textbf{UMLS}}    & \multicolumn{3}{c}{\textbf{KINSHIP}}\\
                        \cmidrule(lr){2-4}  \cmidrule(lr){5-7}                       
                            & MRR         &H@1           &H@3             &MRR         &H@1           &H@3         \\
      \midrule
       DistMult-PL          &0.229         &0.128          &0.243            &0.262        &0.160         &0.284    \\
       DistMult-\ccssl   &\textbf{0.246}&\textbf{0.161} &\textbf{0.265}   &\textbf{0.274}&\textbf{0.170}&\textbf{0.299}    \\
      \midrule
      ComplEx-PL            &\textbf{0.282}&0.160          &\textbf{0.358}  &0.333         &0.226         &0.382 \\
      ComplEx-\ccssl      &0.253         &\textbf{0.191} &0.261           &\textbf{0.344}&\textbf{0.255}&\textbf{0.381}     \\
     \midrule
      QMult-PL              &0.269         &0.157          &\textbf{0.309}  &0.323         &0.228         &0.351   \\
      QMult-\ccssl        &\textbf{0.294}&\textbf{0.217} &\textbf{0.309}  &\textbf{0.328}&\textbf{0.238}&\textbf{0.363}\\     
\bottomrule
\label{table:umls_kinship_1vsall}
\end{tabular}
\end{table}
Table~\ref{table:umls_kinship_1vsall_40_10_split} reports the link prediction performances for the 40:10:50 (\ccssl) and 40:60 (PL) split ratio. Interestingly, the results do not vary much (only in 3 out of 18 cases). This confirms that the labeled data split has critical influence on the overall results in KGE link prediction, whereas the contribution of the self-supervised part is limited. As said before, semi-supervised learning in knowledge graph embedding is much more depending on the training data compared to image classification.

\begin{table}[htbp]
    \caption{
    Link prediction results on UMLS and KINSHIP. 
    The 40:10:50 split ratio for \ccssl-diff and a 40:60 for pseudo-labelling (PL). 
    Bold entries denote best results per method, dataset and metric.}
    \centering
    \begin{tabular}{l  c c c c c c c c c c c c c c c}
      \toprule
                           &\multicolumn{3}{c}{\textbf{UMLS}}    & \multicolumn{3}{c}{\textbf{KINSHIP}}\\
                        \cmidrule(lr){2-4}  \cmidrule(lr){5-7}                       
                            & MRR         &H@1           &H@3             &MRR         &H@1           &H@3         \\
      \midrule
       DistMult-PL          &0.229         &0.128          &0.243            &0.262        &0.160         &0.284    \\
       DistMult-\ccssl    &\textbf{0.246}&\textbf{0.161} &\textbf{0.265}   &\textbf{0.275}&\textbf{0.170}&\textbf{0.299}    \\
      \midrule
      ComplEx-PL            &\textbf{0.282}&0.160          &\textbf{0.358}  &0.333         &0.226         &0.382 \\
      ComplEx-\ccssl      &0.253         &\textbf{0.191} &0.261           &\textbf{0.335}&\textbf{0.232}&\textbf{0.378}     \\
     \midrule
      QMult-PL              &0.269         &0.157          &\textbf{0.309}  &0.323         &0.228         &.351   \\
      QMult-\ccssl        &\textbf{0.294}&\textbf{0.217} &\textbf{0.309}  &\textbf{0.328}&\textbf{0.238}&\textbf{0.363}\\     
\bottomrule
\label{table:umls_kinship_1vsall_40_10_split}
\end{tabular}
\end{table}
\begin{table}[htbp]
    \caption{
    Link prediction results on UMLS and KINSHIP. 
    The 30:20:50 split ratio for \ccssl-diff and a 50:50 for pseudo-labelling (PL). 
    Bold entries denote best results per method, dataset and metric.}
    \centering
    \begin{tabular}{l  c c c c c c c c c c c c c c c}
      \toprule
                           &\multicolumn{3}{c}{\textbf{UMLS}}    & \multicolumn{3}{c}{\textbf{KINSHIP}}\\
                        \cmidrule(lr){2-4}  \cmidrule(lr){5-7}                       
                            & MRR         &H@1          &H@3            &MRR         &H@1           &H@3         \\
      \midrule
       DistMult-PL          &\textbf{0.263}&\textbf{0.164}&\textbf{0.268}  &\textbf{0.302}&\textbf{0.198}&\textbf{0.328}    \\
       DistMult-\ccssl    &0.240         &0.161         &0.253           &0.245         &0.140         &0.267    \\
      \midrule
      ComplEx-PL            &\textbf{0.308}&\textbf{0.205}&\textbf{0.351}  &\textbf{0.392}&\textbf{0.302}&\textbf{0.431}\\
      ComplEx-\ccssl      &0.228         &0.159         &0.241           &0.255         &0.148         &0.285    \\
     \midrule
      QMult-PL              &\textbf{0.330}&\textbf{0.223}&\textbf{0.362}  &\textbf{0.381}&\textbf{0.292}&\textbf{0.422}\\
      QMult-\ccssl       &0.224         &0.150         &0.225           &0.221         &0.118         &0.242    \\
\bottomrule
\label{table:umls_kinship_1vsall_30_20_split}
\end{tabular}
\end{table}

Motivated by the findings in the results shown in \Cref{table:umls_kinship_1vsall,table:umls_kinship_1vsall_40_10_split}, we reduced the training set for \ccssl{} by $25\%$ and increased the training set for pseudo-labeling by $25\%$. 
This setting leads to better generalization performance for PL in all metrics, again giving further evidence for our reasoning about the splits. This case clearly demonstrates a limitation of our method: \ccssl{} is especially favorable in settings where a sufficient amount of labeled instances are provided, particularly when facing structural data such as knowledge graphs.

\section{Generalized Credal Learning: Theoretical Results}

In this section, we provide theoretical results on generalized credal learning as introduced in \cref{sec:ccssl:mit_cb} of the main paper.

\subsection{Proof of Theorem 1}

\begin{figure}[htbp]
    \centering
    \begin{minipage}{.49\textwidth}
        \centering
        \includegraphics[width=\linewidth]{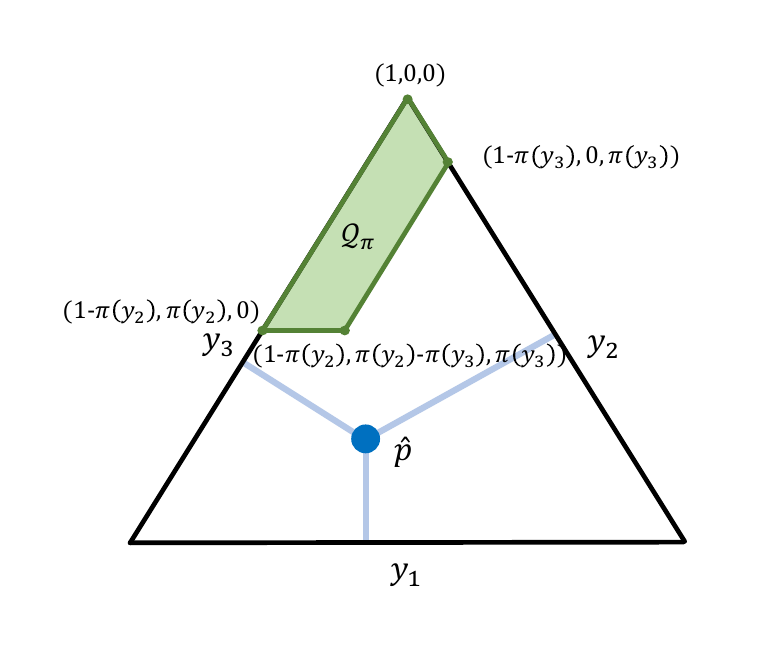}
    \end{minipage}
    \begin{minipage}{.49\textwidth}
        \centering
        \includegraphics[width=\linewidth]{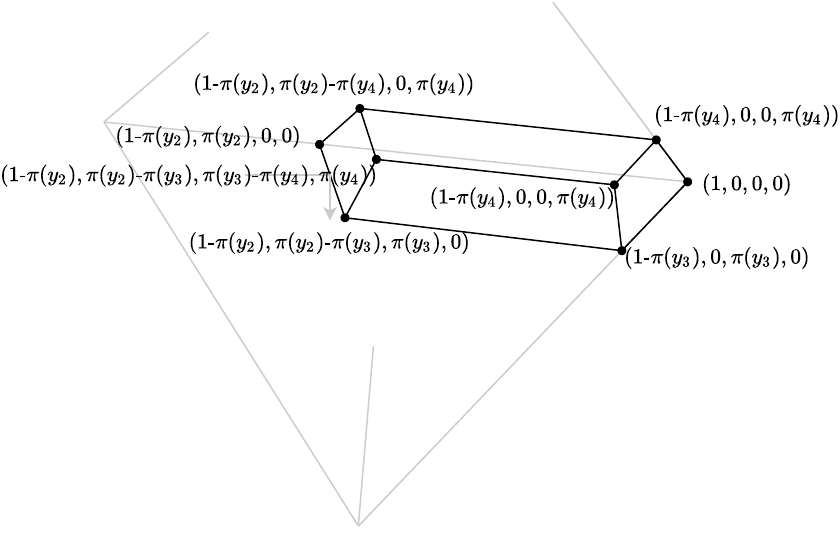}
    \end{minipage}
    \caption{Schematic illustration of a credal set $\mathcal{Q}_\pi$ as a convex polytope in a barycentric coordinate space of all distributions $\mathbb{P}(\mathcal{Y})$ for three and four classes. $\mathcal{Q}_\pi$ is induced by a normalized possibility distribution $\pi$ with $1 = \pi(y_1) \geq \ldots \geq \pi(y_K) \geq 0$.}
    \label{fig:credal_set}
\end{figure}

In the following, we consider possibility distributions $\pi : \mathcal{Y} \fromto [0,1]$, where $\pi_i := \pi(y_i)$ abbreviates the possibility of class $y_i$. Without loss of generality, we assume the possibilities be ordered and normalized, i.e., $0 \leq \pi_1 \leq \ldots \leq \pi_K = 1$ for $K$ classes. Furthermore, we also denote the respective probability of a class $y_i$ given a distribution $p \in \mathbb{P}(\mathcal{Y})$ by $p_i := p(y_i)$.

As described in \cref{sec:ccssl:mit_cb}, the set of inequalities that defines the boundary of a credal set $\mathcal{Q}_\pi$ induces a convex polytope. In \cref{fig:credal_set}, such a credal set is illustrated for three and four classes in a barycentric visualization. The extreme points are marked with the respective probabilities.

In \cref{alg:gen_credal_learning}, we provide an algorithm to solve the problem of finding the closest point in the convex polytope $\mathcal{Q}_\pi$ to a query distribution $\hatp$. In order to proof \cref{theorem:optimality} that states the optimality of this approach, we will introduce the following three lemmas:
\begin{enumerate}
    \item Termination
    \item Optimal projection
    \item Optimal face
\end{enumerate}

\begin{lemma}[Termination]
    Given a normalized possibility distribution $\pi : \mathcal{Y} \fromto [0,1]$, Algorithm 1 terminates for an arbitrary probability distribution $\hatp \in \mathbb{P}(\mathcal{Y})$.
\end{lemma}

\begin{proof}
    In case of $\hatp \in \mathcal{Q}_\pi$, we immediately return with a result.
    For $\hatp \not\in \mathcal{Q}_\pi$, we fix the probabilities of all $y \in Y'$ in each iteration of Algorithm\ 1, which are then removed from $Y$. Thereby, at least for the class $\bar{y} \in Y$ with smallest possibility $\pi(\bar{y})$ 
    \begin{equation*}
        \bar{p}(\bar{y}) = \left(\pi(\bar{y}) - \sum_{y' \not \in Y} p^r(y') \right) \cdot \frac{\hatp(\bar{y})}{\hatp(\bar{y})} \leq \pi(\bar{y})
    \end{equation*}
    holds, which does not violate the possibility constraints as in \cref{eq:possibility_constraints}. Here, the set of classes whose probabilities $p^r$ are set is given by $Y' = \{\bar{y}\}$. Consequently, at least one element in $Y$ is removed per step, which eventually results in an empty set $Y$ and the termination of \cref{alg:gen_credal_learning}.
\end{proof}

In the next lemma, we characterize the optimality of a projection according to the distribution $p^r \in \mathbb{P}(\mathcal{Y})$ as determined in \cref{alg:gen_credal_learning}. In this course, $\bar{Y} \subseteq \mathcal{Y}$ denotes the set of arbitrary, but already fixed probabilities $p^r(y)$ for $y \in \bar{Y}$. This set shall represent classes with optimal probability scores determined in previous iterations. Without loss of generality, we assume that $\forall y \in \bar{Y} : \pi(y) \leq \min_{y' \in \mathcal{Y} \setminus \bar{Y}} \pi(y')$. Moreover, for a certain possibility degree $\pi_i$, let us define the set of classes with at most $\pi_i$ possibility as 
\begin{equation}
    Y_{\pi_i} := \{y \in \mathcal{Y} \given \pi(y) \leq \pi_i\} \enspace .
\end{equation}
For the remaining classes in $\mathcal{Y} \setminus \bar{Y}$, let $p^r$ be constructed as follows:
\begin{equation}
\label{eq:p_r}
    p^r(y) = \begin{cases}
               \left(\pi_i - \sum_{y' \in \bar{Y}} p^r(y') \right) \cdot \frac{\hatp(y)}{\sum_{y' \in Y_{\pi_i} \setminus \bar{Y}} \hatp(y')} & \text{if } y \in Y_{\pi_i} \setminus \bar{Y} \\
                \left(1 - \pi_i \right) \cdot \frac{\hatp(y)}{\sum_{y' \in \mathcal{Y} \setminus Y_{\pi_i}} \hatp(y')} & \text{if } y \in \mathcal{Y} \setminus Y_{\pi_i}
                \end{cases} \, .
\end{equation}

Moreover, we use the notion of a \textit{half space} associated with a possibility constraint $\pi_i$, which is defined as follows.

\begin{definition}[Half-space]\label{def:half_space}
  For a possibility constraint $\pi_i$ of a (normalized) possibility distribution $\pi : \mathcal{Y} \fromto [0,1]$ with $\max_{y \in \mathcal{Y}} \pi(y) = 1$, we define
  \begin{equation*}
      \text{half-space}_{\pi_i} := \bigg\{ p \in \mathbb{P}(\mathcal{Y}) \given \sum_{y \in \mathcal{Y} : \pi(y) \leq \pi_i} p(y) = \pi_i\bigg\}
  \end{equation*}
  as half-space associated with $\pi_i$.
\end{definition}

Given distributions $p^r$ of the form (\ref{eq:p_r}) for a possibility constraint $\pi_i$, which is by construction element of $\text{half-space}_{\pi}$ (due to $\forall y \in \bar{Y} : \pi(y) \leq \pi_i$), we can make the following statement about its optimality.

\begin{lemma}[Optimal projection]\label{lemma:optimal_projection}
    Given a set $\bar{Y} \subseteq \mathcal{Y}$ of classes with arbitrarily fixed probabilities, a (normalized) possibility distribution $\pi : \mathcal{Y} \fromto [0,1]$ with $\max_{y \in \mathcal{Y}} \pi(y) = 1$ and the set $\text{half-space}_{\pi_i}$, the projection $p^r(y) \in \text{half-space}_{\pi_i}$ as defined before is optimal in the sense that $\nexists p \in \text{half-space}_{\pi_i}$ with $p(y) = p^r(y) \, \forall y \in \bar{Y}$ for which $\exists y \in Y_{\pi_i} \setminus \bar{Y}: p(y) \neq p^r(y)$  such that $\kldiv{p}{\hat{p}} < \kldiv{p^r}{\hat{p}}$.
\end{lemma}

\begin{proof}
    Let us define $A := \sum_{y \in \bar{Y}} p^r(y)$ as the sum of the (previously) fixed probabilities. From the definition of $p^r$, we know:
    
    \begin{align}
    \label{eq:simplification}
\begin{split}
        \kldiv{p^r}{\hat{p}} &= \sum_{y \in \mathcal{Y} \setminus Y_{\pi_i}} \frac{(1- \pi_i) \hat{p}(y)}{\sum_{y' \in \mathcal{Y} \setminus Y_{\pi_i}} \hat{p}(y')} \log \frac{\frac{(1- \pi_i) \hat{p}(y)}{\sum_{y' \in \mathcal{Y} \setminus Y_{\pi_i}} \hat{p}(y')}}{\hatp(y)} \\
        & \quad + \sum_{y \in Y_{\pi_i} \setminus \bar{Y}} \frac{(\pi_i - A) \hatp(y)}{\sum_{y' \in Y_{\pi_i} \setminus \bar{Y}} \hatp(y')} \log \frac{\frac{(\pi_i - A) \hatp(y)}{\sum_{y' \in Y_{\pi_i} \setminus \bar{Y}} \hatp(y')}}{\hatp(y)} \\
        & \quad + \sum_{y \in \bar{Y}} p^r(y) \log \frac{p^r(y)}{\hatp(y)} \\
        &= \frac{(1- \pi_i)}{\sum_{y' \in \mathcal{Y} \setminus Y_{\pi_i}} \hat{p}(y')} \left( \sum_{y \in \mathcal{Y} \setminus Y_{\pi_i}} \hat{p}(y) \right) \log \frac{(1- \pi_i)}{\sum_{y' \in \mathcal{Y} \setminus Y_{\pi_i}} \hat{p}(y')} \\
        & \quad + \frac{(\pi_i - A)}{\sum_{y' \in Y_{\pi_i} \setminus \bar{Y}} \hatp(y')} \left( \sum_{y \in Y_{\pi_i} \setminus \bar{Y}} \hatp(y) \right) \log \frac{(\pi_i - A)}{\sum_{y' \in Y_{\pi_i} \setminus \bar{Y}} \hatp(y')} \\
        & \quad + \sum_{y \in \bar{Y}} p^r(y) \log \frac{p^r(y)}{\hatp(y)} \\
        &= (1- \pi_i) \log \frac{(1- \pi_i)}{\sum_{y' \in \mathcal{Y} \setminus Y_{\pi_i}} \hat{p}(y')} + (\pi_i - A) \log \frac{(\pi_i - A)}{\sum_{y' \in Y_{\pi_i} \setminus \bar{Y}} \hatp(y')} \\
        & \quad + \sum_{y \in \bar{Y}} p^r(y) \log \frac{p^r(y)}{\hatp(y)} \\
        \end{split}
    \end{align}
    
    Now, suppose $\exists p \in \text{half-space}_{\pi_i}$ with $p(y) = p^r(y) \, \forall y \in \bar{Y}$ for which $\exists y \in Y_{\pi_i} \setminus \bar{Y}: p(y) \neq p^r_{\pi_i}(y)$ such that $\kldiv{p}{\hat{p}} < \kldiv{p^r}{\hat{p}}$, which would lead to a contradiction of the lemma.
    
    \begin{align*}
        & \quad \kldiv{p}{\hatp} - \kldiv{p^r}{\hatp} \\
        &= \sum_{y \in \mathcal{Y} \setminus Y_{\pi_i}} p(y) \log \frac{p(y)}{\hatp(y)} + \sum_{y \in Y_{\pi_i} \setminus \bar{Y}} p(y) \log \frac{p(y)}{\hatp(y)} + \sum_{y \in \bar{Y}} p(y) \log \frac{p(y)}{\hatp(y)} \\
        & \quad - \sum_{y \in \mathcal{Y} \setminus Y_{\pi_i}} p^r(y) \log \frac{p^r(y)}{\hatp(y)} - \sum_{y \in Y_{\pi_i} \setminus \bar{Y}} p^r(y) \log \frac{p^r(y)}{\hatp(y)} - \sum_{y \in \bar{Y}} p^r(y) \log \frac{p^r(y)}{\hatp(y)} \\
        &= \sum_{y \in \mathcal{Y} \setminus Y_{\pi_i}} p(y) \log \frac{p(y)}{\hatp(y)} + \sum_{y \in Y_{\pi_i} \setminus \bar{Y}} p(y) \log \frac{p(y)}{\hatp(y)} \\
        & \quad - \sum_{y \in \mathcal{Y} \setminus Y_{\pi_i}} p^r(y) \log \frac{p^r(y)}{\hatp(y)} - \sum_{y \in Y_{\pi_i} \setminus \bar{Y}} p^r(y) \log \frac{p^r(y)}{\hatp(y)}\\
        & \geq \left( \sum_{y \in \mathcal{Y} \setminus Y_{\pi_i}} p(y) \right) \log \frac{\left( \sum_{y \in \mathcal{Y} \setminus Y_{\pi_i}} p(y) \right)}{\left( \sum_{y \in \mathcal{Y} \setminus Y_{\pi_i}} \hatp(y) \right)} + \left( \sum_{y \in Y_{\pi_i} \setminus \bar{Y}} p(y) \right) \log \frac{\left( \sum_{y \in Y_{\pi_i} \setminus \bar{Y}} p(y) \right)}{\left( \sum_{y \in Y_{\pi_i} \setminus \bar{Y}} \hatp(y) \right)} \\
        & \quad - \sum_{y \in \mathcal{Y} \setminus Y_{\pi_i}} p^r(y) \log \frac{p^r(y)}{\hatp(y)} - \sum_{y \in Y_{\pi_i} \setminus \bar{Y}} p^r(y) \log \frac{p^r(y)}{\hatp(y)}\\
    \end{align*}
    
    Here, the last equation can be derived from the log sum inequality.\footnote{For $a_1, \dots, a_n, b_1, \dots, b_n \in \mathbb{R}_+$, the log sum inequality states that $\sum_{i\in[n]} a_i \log \frac{a_i}{b_i} \geq a \log \frac{a}{b}$, whereby $a:= \sum_{i\in[n]} a_i$ and $b:=\sum_{i\in[n]} b_i$.} As we are projecting onto the half space associated with $\pi_i$, we can further see that $\sum_{y \in \mathcal{Y} \setminus Y_{\pi_i}} p(y) = 1 - \pi_i$ and $\sum_{y \in Y_{\pi_i} \setminus \bar{Y}} p(y) = \pi_i - A$, which also applies to $p^r$ for the same class subsets.
    
    As a result, together with (\ref{eq:simplification}), we get
    
    \begin{align*}
        &= \left( 1 - \pi_i \right) \log \frac{\left( 1 - \pi_i \right)}{\left( \sum_{y \in \mathcal{Y} \setminus Y_{\pi_i}} \hatp(y) \right)} + \left( \pi_i - A \right) \log \frac{\left( \pi_i- A \right)}{\left( \sum_{y \in Y_{\pi_i} \setminus \bar{Y}} \hatp(y) \right)} \\
        & \quad - \left( 1 - \pi_i \right) \log \frac{\left( 1 - \pi_i \right)}{\left( \sum_{y \in \mathcal{Y} \setminus Y_{\pi_i}} \hatp(y) \right)} - \left( \pi_i - A \right) \log \frac{\left( \pi_i- A \right)}{\left( \sum_{y \in Y_{\pi_i} \setminus \bar{Y}} \hatp(y) \right)} = 0 \nless 0 \enspace ,
    \end{align*}
    
    which leads to a contradiction.
\end{proof}

For the next lemma, we introduce the notion of a \textit{face} as follows:

\begin{definition}[Face]\label{def:face}
  For a possibility constraint $\pi_i$ of a (normalized) possibility distribution $\pi : \mathcal{Y} \fromto [0,1]$ and distributions $p \in \mathbb{P}(\mathcal{Y})$ with $\sum_{y \in \mathcal{Y} : \pi(y) \leq \pi_i} p(y) = \pi_i$ (i.e., $p \in \text{half-space}_{\pi_i}$), we define
  \begin{equation*}
      \text{face}_{\pi_i} := \bigg\{ p \given \sum_{k=1}^j p_k \leq \pi_j , \quad j \in \{1, ..., i\} \bigg\}
  \end{equation*}
  as face associated with $\pi_i$.
\end{definition}

Effectively, $\text{face}_{\pi_i}$ considers the subspace of $\text{half-space}_{\pi_i}$ that does not violate any of the possibility constraints $\pi_j \leq \pi_i$. One can readily see that each iteration of Algorithm\ 1 determines the highest possibility constraint $\pi(y^*)$ whose projection $p^r$ as defined before is element of $\text{face}_{\pi(y^*)}$. Given this fact, we show the optimality of this selected face in the next lemma.

\begin{lemma}[Optimal face]
    Given a set $\bar{Y} \subseteq \mathcal{Y}$ of classes with arbitrarily fixed probabilities, a (normalized) possibility distribution $\pi : \mathcal{Y} \fromto [0,1]$ with $\max_{y \in \mathcal{Y}} \pi(y) = 1$ and an arbitrary distribution $\phat \notin \mathcal{Q}_\pi$, Algorithm\ 1 selects  $\text{face}_{\pi_i}$ in each iteration that is optimal in the sense that $\nexists j \neq i$ with $p^* \in \argmin_{p \in \text{face}_{\pi_j}} \kldiv{p}{\hatp}$, $p^*(y) = p^r(y) \, \forall y \in \bar{Y}$ and $\exists y \in Y_{\pi_i} \setminus \bar{Y} : p^*(y) \neq p^r(y)$, such that $\kldiv{p^*}{\hatp} < \kldiv{p^r}{\hatp}$ for $p^r \in \argmin_{p \in \text{face}_{\pi_i}} \kldiv{p}{\hatp}$.
\end{lemma}

\begin{proof}
    Again, we define $A := \sum_{y \in \bar{Y}} p^r(y)$. Now, let us assume $\exists j \neq i$ with the properties described in this lemma, which leads us to a contradiction. To proof it, we can distinguish three cases. 
    
    \begin{enumerate}[label=Case \arabic*:]
        \item $\pi_j > \pi_i$. It is easy to see that a violation with respect to \cref{eq:possibility_constraints} for possibility $\leq \pi_j$ exists, i.e., the projection constructed as in (\ref{eq:p_r}) for $\pi_j$ is not element of $\text{face}_{\pi_j}$. In this case, the distribution $p^* \in \argmin_{\text{face}_{\pi_j}} \kldiv{p}{\hatp}$ is located on an ``edge'' of the face, i.e., $\exists m : \sum_{k=1}^m p^*_k = \pi_m$.
        
        If $m=i$, then Lemma\ \ref{lemma:optimal_projection} implies that $p^*(y) = p^r(y) \, \forall y \in Y_{\pi_i}$, which contradicts the assumptions.
        
        In case of $m \neq i$, we can derive the following results using a similar scheme as in the proof of Lemma\ \ref{lemma:optimal_projection}:
        
        \begin{align*}
            & \quad \kldiv{p^*}{\hatp} - \kldiv{p^r}{\hatp} \\
            &= \sum_{y \in \mathcal{Y} \setminus Y_{\pi_i}} p^*(y) \log \frac{p^*(y)}{\hatp(y)} + \sum_{y \in Y_{\pi_i} \setminus \bar{Y}} p^*(y) \log \frac{p(y)}{\hatp(y)} + \sum_{y \in \bar{Y}} p(y) \log \frac{p(y)}{\hatp(y)} \\
            & \quad - \sum_{y \in \mathcal{Y} \setminus Y_{\pi_i}} p^r(y) \log \frac{p^r(y)}{\hatp(y)} - \sum_{y \in Y_{\pi_i} \setminus \bar{Y}} p^r(y) \log \frac{p^r(y)}{\hatp(y)} - \sum_{y \in \bar{Y}} p^r(y) \log \frac{p^r(y)}{\hatp(y)} \\
            & \geq \left( \sum_{y \in \mathcal{Y} \setminus Y_{\pi_i}} p^*(y) \right) \log \frac{\left( \sum_{y \in \mathcal{Y} \setminus Y_{\pi_i}} p^*(y) \right)}{\left( \sum_{y \in \mathcal{Y} \setminus Y_{\pi_i}} \hatp(y) \right)} + \left( \sum_{y \in Y_{\pi_i} \setminus \bar{Y}} p^*(y) \right) \log \frac{\left( \sum_{y \in Y_{\pi_i} \setminus \bar{Y}} p^*(y) \right)}{\left( \sum_{y \in Y_{\pi_i} \setminus \bar{Y}} \hatp(y) \right)} \\
            & \quad - \sum_{y \in \mathcal{Y} \setminus Y_{\pi_i}} p^r(y) \log \frac{p^r(y)}{\hatp(y)} - \sum_{y \in Y_{\pi_i} \setminus \bar{Y}} p^r(y) \log \frac{p^r(y)}{\hatp(y)}\\
        \end{align*}
        
        As we know that $p^*$ does not violate any possibility constraints in $Y_{\pi_j} \supset Y_{\pi_i}$ due to $p^* \in \text{face}_{\pi_j}$, but is also not on the same edge as implied by $\pi_i$, it holds that 
        \begin{equation*}
            \sum_{y \in Y_{\pi_i} \setminus \bar{Y}} p^*(y) < \pi_i - A \enspace .
        \end{equation*}
        
        Moreover, as we know that there is no possibility violation by $p^r(y) \, \forall \in Y_{\pi_i}$ associated with $\pi$ (cf. Lemma\ \ref{lemma:optimal_projection}), it must hold $\sum_{Y_{\pi_i} \setminus \bar{Y}} \hatp(y) > \pi_i - A$. Otherwise, $\hatp$ would be element of $\mathcal{Q}_\pi$.
        
        Altogether, one can readily follow
        \begin{align}
        \label{eq:lemma3}
        \begin{split}
            &= \underbrace{\left( \sum_{y \in \mathcal{Y} \setminus Y_{\pi_i}} p^*(y) \right)}_{> 1 - \pi_i} \log \frac{ \overbrace{\left( \sum_{y \in \mathcal{Y} \setminus Y_{\pi_i}} p^*(y) \right)}^{> 1- \pi_i}}{\left( \sum_{y \in \mathcal{Y} \setminus Y_{\pi_i}} \hatp(y) \right)} + \underbrace{\left( \sum_{y \in Y_{\pi_i} \setminus \bar{Y}} p^*(y) \right)}_{< \pi_i - A} \log \frac{\overbrace{\left( \sum_{y \in Y_{\pi_i} \setminus \bar{Y}} p^*(y) \right)}^{< \pi_i - A}}{\left( \sum_{y \in Y_{\pi_i} \setminus \bar{Y}} \hatp(y) \right)} \\
            & \quad - (1 - \pi_i) \log \frac{(1 - \pi_i)}{\sum_{y' \in \mathcal{Y} \setminus Y_{\pi_i} }\hatp(y')} - (\pi_i - A) \log \frac{(\pi_i - A)}{\sum_{y' \in Y_{\pi_i} \setminus \bar{Y}} \hatp(y')} > 0 \nless 0 \enspace ,
        \end{split}
        \end{align}
        leading to a contradiction.
        
        \item $\pi_j = \pi_i$. Lemma\ \ref{lemma:optimal_projection} implies the optimality of $p^r$ in this case.
        \item $\pi_j < \pi_i$. Here, one can again distinguish whether there exists a violation of $p^r$ constructed for $\pi_j$ or not. In the first case, we can apply exactly the same idea as before by showing that the projection $p^* \in \argmin_{p \in \text{face}_{\pi_j}} \kldiv{p}{\hatp}$ is located on an edge that is associated with a $\pi_m$ with $m< i$.
        
        In case there is no violation, we know that $\sum_{k=1}^j p^*_k = \pi_j$. This leads to $\sum_{k=1}^i p^*_k \leq \pi_i$. Together with $\sum_{Y_{\pi_i} \setminus \bar{Y}} \hatp(y) > \pi_i - A$, one can derive a similar equation as in Lemma\ \ref{eq:lemma3} to show that 
        \begin{equation*}
            \kldiv{p^*}{\hatp} - \kldiv{p^r}{\hatp} \geq 0 \nless 0 \enspace ,
        \end{equation*}
        leading again to a contradiction.
    \end{enumerate}
\end{proof}

By combining the previous results, we are ready to proof the optimality of Algorithm\ 1.

\setcounter{theorem}{0}

\begin{theorem}[Optimality]\label{theorem:optimality} 
Given a credal set $\mathcal{Q}_\pi$ induced by a normalized possibility distribution $\pi : \mathcal{Y} \fromto [0,1]$ with $\max_{y \in \mathcal{Y}} \pi(y) = 1$ according to (1), Algorithm 1 returns the solution of $\mathcal{L}^*(\mathcal{Q}_\pi, \hat{p})$ as defined in (11) for an arbitrary distribution $\hat{p} \in \mathbb{P}(\mathcal{Y})$.
\end{theorem}

\begin{proof}
    Combining the three previous lemmas, as well as the fact that the solution of $\argmin_{p \in \mathcal{Q}_\pi} \kldiv{p}{\hatp}$ is always characterized by an extreme point on one of the faces of the convex polytope $\mathcal{Q}_\pi$ for $\hatp \not \in \mathcal{Q}_\pi$, leads to \cref{theorem:optimality}: In each iteration, we choose the optimal projection on the optimal face. Thus, we maintain the optimal probabilities $p^r(y)$ for all $y \not\in Y$.
    
    In case of $\hatp \in \mathcal{Q}_\pi$, Algorithm 1 returns $\kldiv{\hatp}{\hatp} = 0$, which is optimal by definition of $D_{KL}$.
\end{proof}

\subsection{Complexity}

The complexity of Algorithm\ 1 requires a proper specification of how $y^*$ is determined in the while loop. In our implementation, we sort the classes $y \in \mathcal{Y}=\{y_1, \ldots, y_K\}$ according to their possibilities $\pi(y)$ in a descending manner first, which can be done in $\mathcal{O}(K \log K)$. Then, the while loop iterates over the sorted classes and can continue with the next while-loop until a matching constraint $\pi(y^*)$ could be determined. This violation check involves iterating over all classes $y \in Y$ with $\pi(y) \leq \pi(y^*)$. By sorting the elements in advance, this becomes efficient.

\paragraph{Worst-Case Complexity} In the worst-case, every iteration of Algorithm\ 1 has to iterate over all remaining elements in $Y$. As said before, checking violations of the possibility constraints requires iterating over all involved classes. Thus, the worst-case complexity can be (loosely) bounded by
\begin{equation*}
    \sum_{i=0}^{K-1} (i + 1)(K - i) = \frac{K^3}{6} + \frac{K^2}{2} + \frac{K}{3} = \mathcal{O}(K^3) \enspace .
\end{equation*}

\paragraph{Average-Case Complexity} Although we are not providing a rigorous analysis of the average-case complexity here, we characterize the efficiency of our algorithm in several cases.

We can observe that the worst case applies whenever the projection of a query distribution $\hatp \not\in \mathcal{Q}_\pi$ on the convex polytope $\mathcal{Q}_\pi$ is the distribution $p^*$ with $p^*(y_i) = \pi(y_i) - \pi(y_{i-1})$ for all $i \in \{1, \ldots, K-1\}$ and $p^*(y_K) = \pi(y_K)$ for sorted classes $y_i$ according to their possibilities. This is the case when $\hatp$ is in the cone associated with this extreme point $p^*$ \citep{https://doi.org/10.48550/arxiv.2201.10161}, which is given by
\begin{equation*}
    \bigg\{ p \not\in \mathcal{Q}_\pi \given p^* \in \argmin_{p' \in \mathcal{Q}_\pi} \kldiv{p'}{p} \bigg\} \enspace .\footnote{More precisely, one would have to distinguish the cases where a $p^r$ projection as in \cref{eq:p_r} is perfectly matching the extreme point $p^*$, but we omit it here for simplicity.} 
\end{equation*}
This set, however, depends on the size of the faces resp. the credal set and is typically rather small. Moreover, it gets proportionally smaller with higher values of $K$.

In the (trivial) case of $\hatp \in \mathcal{Q}_\pi$, we achieve linear complexity $\mathcal{O}(K)$ as we have to check the possibility constraints for all $K$ classes only once. In the other cases, the complexity depends on the face on which we need to project $\hatp$: When projecting on $\text{face}_{\pi_i}$, we do not have to consider any class $y$ with $\pi(y) \leq \pi_i$. Roughly speaking, the larger the faces associated with high possibilities become, the higher the chance of (optimally) projecting directly on this face and not requiring any loop iterations over classes with smaller possibility values, leading to a sublinear amount of face projections and thus reducing the cubic complexity.

\end{document}